\documentclass[11pt,letterpaper,english]{article}

\usepackage[T1]{fontenc}
\usepackage[utf8]{inputenc}
\usepackage[english]{babel}
\usepackage[margin=1in]{geometry}
\usepackage{natbib}
\usepackage{authblk}
\usepackage{macros}

\usepackage{times}
\allowdisplaybreaks
\begin{document}

\title{On TD(0) with function approximation: Concentration bounds and a centered variant with exponential convergence}

\author[1]{Nathaniel Korda\thanks{nathaniel.korda@eng.ox.ac.uk}}
\author[2]{Prashanth L.A.\thanks{prashla@isr.umd.edu}}

\affil[1]{\small MLRG, University of Oxford, UK}
\affil[2]{\small Institute for Systems Research, University of Maryland, USA}

\renewcommand\Authands{ and }

\date{}
\maketitle
%%%%%%%%%%%%%%%%%%%%%%%%%%%%%%%%%%%%%%%%%%%%%%%%%%%%%%%%%%%%%%
%%%%%%%%%%%%%%%%%%%%%%%%%%%%%%%%%%%%%%%%%%%%%%%%%%%%%%%%%%%%%%
%%%%%%%%%%%%%%%%%%%%%%%%%%%%%%%%%%%%%%%%%%%%%%%%%%%%%%%%%%%%%%
%%%%%%%%%%%%%%%%%%%%%%%%%%%%%%%%%%%%%%%%%%%%%%%%%%%%%%%%%%%%%%
\begin{abstract} 
We provide non-asymptotic bounds for the well-known temporal difference learning algorithm TD(0) with linear function approximators. These include high-probability bounds as well as bounds in expectation. Our analysis suggests that a step-size inversely proportional to the number of iterations cannot guarantee optimal rate of convergence unless we assume (partial) knowledge of the stationary distribution for the Markov chain underlying the policy considered. We also provide bounds for the iterate averaged TD(0) variant, which gets rid of the step-size dependency while exhibiting the optimal rate of convergence. Furthermore, we propose a variant of TD(0) with linear approximators that incorporates a centering sequence, and  establish that it exhibits an exponential rate of convergence in expectation. We demonstrate the usefulness of our bounds on two synthetic experimental settings.
\end{abstract} 

%%%%%%%%%%%%%%%%%%%%%%%%%%%%%%%%%%%%%%%%%%%%%%%%%%%%%%%%%%%%%%
%%%%%%%%%%%%%%%%%%%%%%%%%%%%%%%%%%%%%%%%%%%%%%%%%%%%%%%%%%%%%%
%%%%%%%%%%%%%%%%%%%%%%%%%%%%%%%%%%%%%%%%%%%%%%%%%%%%%%%%%%%%%%
%%%%%%%%%%%%%%%%%%%%%%%%%%%%%%%%%%%%%%%%%%%%%%%%%%%%%%%%%%%%%%
%%%%%%%%%%%%%%%%%%%%%%%%%%%%%%%%%%%%%%%%%%%%%%%%%%%%%%%%%%%%%%
\section{Introduction}
Many stochastic control problems can be cast within the framework of Markov decision processes (MDP).
Reinforcement learning (RL) is a popular approach to solve MDPs, when the underlying transition mechanism is unknown.
An important problem in RL is to estimate the \textit{value function} $V^\pi$ for a given stationary policy $\pi$.
We focus on discounted reward MDPs with a high-dimensional state space $\S$.
In this setting, one can only hope to estimate the value function approximately and this constitutes the {\em policy evaluation} step in several approximate policy iteration methods, e.g. actor-critic algorithms  \citep{konda2003onactor}, \citep{bhatnagar2009natural}. 

\textbf{Temporal difference learning}  is a well-known policy evaluation algorithm that is both online and works with a single sample path obtained by simulating the underlying MDP. 
However, the classic TD(0) algorithm uses full-state representations (i.e. it stores an entry for each state $s \in \S$) and hence, suffers from the curse of dimensionality.
A standard trick to alleviate this problem is to approximate the value function within a linearly parameterized space of functions, i.e., $V^\pi(s) \approx \theta\tr \phi(s)$. Here $\theta$ is a tunable parameter and $\phi(s)$ is a column feature vector with dimension $d<<|S|$.
This approximation allows for efficient implementation of TD(0) even on large state spaces.

The update rule for TD(0) that incorporates linear function approximators is as follows: Starting with an arbitrary $\theta_0$, 
\begin{align}
\theta_{n+1} = \theta_{n} + \gamma_{n} \big(r(s_n,\pi(s_n)) + \beta \theta_{n}\tr \phi(s_{n+1}) 
- \theta_{n}\tr \phi(s_n)&\big)\phi(s_{n}).\label{eq:td-update-intro}
\end{align}
In the above, the quantities $\gamma_n$ are \emph{step sizes}, chosen in advance ,and satisfying standard stochastic approximation conditions (see assumption (A5)).
Further, $r(s,a)$ is the reward recieved in state $s$ on choosing action $a$, $\beta \in (0,1)$ is a discount factor, and $s_n$ is the state of the MDP at time $n$. 

\paragraph{Asymptotic convergence of TD(0).} In \citep{tsitsiklis1997analysis}, the authors establish that $\theta_n$ governed by \eqref{eq:td-update-intro} converges almost surely to the fixed point, $\theta^*$, of the {\em projected Bellman equation} given by
  \begin{align}
  \Phi \theta^* = \Pi \T^\pi(\Phi \theta^*).
  \label{eq:td-fixed-point-intro}
\end{align}
In the above, $\T^\pi$ is the Bellman operator, $\Pi$ is the orthogonal projection onto the linearly parameterized space within which we approximate the value function, and $\Phi$ is the feature matrix with rows $\phi(s)\tr, \forall s \in \S$ denoting the features corresponding to state $s \in \S$ (see Section \ref{sec:background} for more details).
Let $P$ denote the transition probability matrix with components $p(s,\pi(s),s')$ that denote the probability of transitioning from state $s$ to $s'$ under the action $\pi(s)$. Let $r$ be a vector with components $r(s,\pi(s))$, and $\Psi$ be a diagonal matrix whose diagonal forms the stationary distribution (assuming it exists) of the Markov chain for the underlying policy $\pi$.
Then, $\theta^*$ can be written as the solution to the following system of equations (see Section 6.3 of \citep{bertsekas2011approximate})
\begin{align}
\label{eq:td-system-of-equations}
\hspace{-0.5em}A \theta^* = b, \textrm{ where } A = \Phi\tr \Psi(I- \beta P)\Phi \textrm{ and } b=\Phi\tr \Psi r.
\end{align}

\paragraph{Our work.}
We derive non-asymptotic bounds on $\l \theta_n - \theta^* \r$, both in high-probability and in expectation, to quantify the rate of convergence of  TD(0) with linear function approximators. To the best of our knowledge, there are no non-asymptotic bounds for TD(0) with function approximation, while there are asymptotic convergence and rate results available.

\textbf{\textit{Finite time analysis is challenging because}}:\\
  \begin{inparaenum}[\bfseries(1)]
   \item The asymptotic limit $\theta^*$ is the fixed point of the Bellman operator, which assumes that the underlying MDP is begun from the stationary distribution $\Psi$ (whose influence is evident in \eqref{eq:td-system-of-equations}). However, the samples provided to the algorithm come from simulations of the MDP that are not begun from $\Psi$. This is a problem for a finite time analysis, since we do not know exactly the number of steps after which mixing of the underlying Markov chain has occurred, and the distribution of the samples that TD(0) sees has become the stationary distribution. Moreover, an assumption on this mixing rate amounts to assuming (partial) knowledge of the transition dynamics of the Markov chain underlying the policy $\pi$.\\
   \item Standard results from stochastic approximation theory suggest that in order to obtain the optimal rate of convergence for a step size choice of $\gamma_n = c/(c+n)$, one has to chose the constant $c$ carefully. In the case of TD(0), we derive this condition and point out the optimal choice for $c$ requires knowledge of the mixing rate of the underlying Markov chain for policy $\pi$.\\ 
  \end{inparaenum}
We handle the first problem by establishing that under a mixing assumption (the same as that used to establish asymptotic convergence for TD(0) in \citep{tsitsiklis1997analysis}), the mixing error can be handled in the non-asymptotic bound. This assumption is broad enough to encompass a reasonable range of MDP problems. We alleviate the second problem by using iterate averaging. 

\textbf{\textit{Variance reduction.}} One inherent problem with iterative schemes that use a single sample to update the iterate at each time step, is that of variance. This is the reason why it is necessary to carefully choose the step-size sequence: too large and the variance will force divergence; too small and the algorithm will converge, but not to the solution intended. Indeed, iterate averaging is a technique that aims to allow for larger step-sizes, while producing the same overall rate of convergence (and we show that it succeeds in eliminating the necessity to know properties of the stationary distribution of the underlying Markov chain). A more direct approach is to center the updates, and this was pioneered recently for solving batch problems via stochastic gradient descent in convex optimization \citep{johnson2013accelerating}. We propose a variant of TD(0) that uses this approach, though our setting is considerably more complicated as samples arrive online and the function being optimized is 
not accessible directly. 
%We give a finite-time analysis for the centered TD algorithm, and show that the algorithm results in an exponential convergence rate.
% is not necessary for this algorithm.

\paragraph{Our contributions} can be summarized as follows:\\
\begin{inparaenum}[\bfseries(1)]  
\item \textbf{\textit{Concentration bounds.}}
Under assumptions similar to \citep{tsitsiklis1997analysis}, we provide non-asymptotic bounds, both in high probability as well as in expectation and these quantify the convergence rate of TD(0) with function approximation. \\
\item \textbf{\textit{Centered TD.}} We propose a variant of TD(0) that incorporates a centering sequence and we show that it converges faster than the regular TD(0) algorithm in expectation. 
\end{inparaenum}

The key insights from our finite-time analysis are:\\
\begin{inparaenum}[\bfseries(1)]  
\item Choosing $\gamma_n= \frac{c_0c}{(c+n)}$, with $c_0< \mu(1-\beta)/(2(1+\beta)^2)$ and $c$ such that $\mu (1-\beta)c_0c >1$, we obtain the optimal rate of convergence of the order $O\left(1/\sqrt{n}\right)$, both in high-probability as well as in expectation. Here $\mu$ is the smallest eigenvalue of the matrix $\Phi\tr\Psi\Phi$ (see Theorem \ref{thm:td-rate}). However, obtaining this rate is problematic as it implies (partial) knowledge (via $\mu$) of the transition dynamics of the MDP.\\
\item With iterate averaging, one can get rid of the step-size dependency and still obtain the optimal rate of convergence, both in high probability as well as in expectation (see Theorem \ref{thm:td-avg-rate}).\\
\item For the centered variant of TD(0), we obtain an exponential convergence rate when the underlying Markov chain mixes fast (see Theorem \ref{thm:ctd-bound}).\\
\item We illustrate the usefulness of our bounds on two simple synthetic experimental setups. In particular, using the step-sizes suggested by our bounds in Theorems \ref{thm:td-rate}--\ref{thm:ctd-bound}, we are able to establish convergence empirically for TD(0), and both its averaging, as well as centered variants.
%This is a significant leap (from $O\left(1/n\right)$ to $O\left(\eta^n\right)$, where $\eta\in (0,1)$) in the convergence rate for TD(0) type schemes and this algorithm can be easily incorporated into approximate policy iteration schemes.  
\end{inparaenum}

%TD(0) with function approximation is an efficient algorithm that is easy to implement on large state space problems. However, deriving convergence rate results, especially of non-asymptotic nature, requires sophisticated machinery. In particular, we base our approach on that proposed in \citep{frikha2012concentration} (and later expanded to include iterate averaging in \citep{fathi2013transport}). 
\paragraph{Related work.}
Concentration bounds for general stochastic approximation schemes have been derived in \citep{frikha2012concentration} and later expanded to include iterate averaging in \citep{fathi2013transport}. Unlike the aforementioned reference, deriving convergence rate results for TD(0), especially of non-asymptotic nature, requires sophisticated machinery as it involves Markov noise that impacts the mixing rate of the underlying Markov chain.  An asymptotic normality result for TD($\lambda$) is available in \citep{kondathesis}. The authors establish there that TD($\lambda$) converges asymptotically to a multi-variate Gaussian distribution with a covariance matrix that depends on $A$ (see \eqref{eq:td-system-of-equations}). This rate result holds true for TD($\lambda$) when combined with iterate averaging, while the non-averaged case does not result in the optimal rate of convergence. Our results are consistent with this observation, as we establish from a finite time analysis that the non-averaged TD(0) can result 
in optimal convergence only if the step-size constant $c$ in $\gamma_n=c/(c+n)$ is set carefully (as a function of a certain quantity that depends on the stationary distribution - see (A3) below), while one can get rid of this dependency and still obtain the optimal rate with iterate averaging.
Least squares temporal difference methods are popular alternatives to the classic TD($\lambda$). Asymptotic convergence rate results for LSTD($\lambda$) and LSPE($\lambda$), two popular least squares methods, are available in \citep{kondathesis} and \citep{yu2009convergence}, respectively. However, to the best of our knowledge, there are no concentration bounds that quantify the rate of convergence through a finite time analysis. A related work in this direction is the finite time bounds for LSTD in \citep{lazaric2010finite}. However, the analysis there is under a fast mixing rate assumption, while we provide non-asymptotic rate results without making any such assumption. We note here that assuming a mixing rate implies partial knowledge of the transition dynamics of the MDP under a stationary policy and in typical RL settings, this information is not available.
%%%%%%%%%%%%%%%%%%%%%%%%%%%%%%%%%%%%%%%%%%%%%%%%%%%%%%%%%%%%%%
%%%%%%%%%%%%%%%%%%%%%%%%%%%%%%%%%%%%%%%%%%%%%%%%%%%%%%%%%%%%%%
\section{TD(0) with Linear Approximation}
\label{sec:background}
We consider an MDP with state space $\S$ and action space $\A$. 
The aim is to estimate the value function $V^\pi$ for any given stationary policy $\pi:\S\rightarrow \A$, where
\begin{equation}
V^\pi(s) := \E\left[\sum_{t=0}^{\infty} \beta^t r(s_t,\pi(s_t))\mid s_0=s\right].
\label{eq:vf}
\end{equation}
Recall that $\beta \in (0,1)$ is the discount factor, $s_t$ denotes the state of the MDP at time $t$, and $r(s,a)$ denotes the reward obtained in state $s$ under action $a$. 
The expectation in \eqref{eq:vf} is taken with respect to the transition dynamics $P$. 
It is well-known that $V^\pi$ is the solution to the fixed point relation $V = \T^\pi(V)$, where the Bellman operator $\T^\pi$ is defined as
\begin{align}
\T^\pi(V)(s) := r(s,\pi(s)) +  \beta \sum\limits_{s'} p(s,\pi(s),s') V(s'),
\label{vf}
\end{align}
TD(0) \citep{sutton1998reinforcement} performs a fixed point-iteration using stochastic approximation:
Starting with an arbitrary $V_0$, update
\begin{align}
V_n(s_n) := V_{n-1}(s_n)& + \gamma_n\big(r(s_n,\pi(s_n)) +  \beta V_{n-1}(s_{n+1}) - V_{n-1}(s_n)\big),
\label{eq:td-full-state}
\end{align}
where $\gamma_n$ are step-sizes that satisfy standard stochastic approximation conditions. 
% The iterate $V_n$ can be shown to converge to the true value function $V^\pi$ (see \citep{bertsekas1996neuro} for details). 

As discussed in the introduction, while TD(0) algorithm is simple and provably convergent to the fixed point of $\T^\pi$ for any policy, it suffers from the curse of dimensionality associated with high-dimensional state spaces, and popular method to allieviate this is to parameterize the value function using a linear function approximator, i.e. for every $s \in \S$, approximate $V^\pi(s) \approx  \phi(s)\tr \theta$.
Here $\phi(s)$ is a $d$-dimensional feature vector with $d << |\S|$, and $\theta$ is a tunable parameter.
Incorporating function approximation, an update rule for TD(0) analogous to \eqref{eq:td-full-state} is given in \eqref{eq:td-update-intro}. 
%%%%%%%%%%%%%%%%%%%%%%%%%%%%%%%%%%%%%%%%%%%%%%%%%%%%%%%%%%%%%%
%%%%%%%%%%%%%%%%%%%%%%%%%%%%%%%%%%%%%%%%%%%%%%%%%%%%%%%%%%%%%%

%%%%%%%%%%%%%%%%%%%%%%%%%%%%%%%%%%%%%%%%%%%%%%%%%%%%%%%%%%%%%%
%%%%%%%%%%%%%%%%%%%%%%%%%%%%%%%%%%%%%%%%%%%%%%%%%%%%%%%%%%%%%%
%%%%%%%%%%%%%%%%%%%%%%%%%%%%%%%%%%%%%%%%%%%%%%%%%%%%%%%%%%%%%%
%%%%%%%%%%%%%%%%%%%%%%%%%%%%%%%%%%%%%%%%%%%%%%%%%%%%%%%%%%%%%%
%%%%%%%%%%%%%%%%%%%%%%%%%%%%%%%%%%%%%%%%%%%%%%%%%%%%%%%%%%%%%%
\section{Concentration bounds for TD(0)}
\label{sec:results}
\subsection{Assumptions}
\begin{enumerate}[\bfseries({A}1)]
\item \textit{\textbf{Ergodicity}}: The Markov chain induced by the policy $\pi$ is irreducible and aperiodic. Moreover, there exists a stationary distribution $\Psi(=\Psi_\pi)$ for this Markov chain. Let $\E_\Psi$ denote the expectation w.r.t. this distribution.\\
\item \textit{\textbf{Bounded rewards}}: $|  r(s,\pi(s)) | \le 1,$ for all $s\in\S$.\\
\item \textit{\textbf{Linear independence}}: The feature matrix $\Phi$ has full column rank. This assumption implies that the matrix $\Phi\tr\Psi\Phi$ has smallest eigenvalue $\mu>0$.\\
\item \textit{\textbf{Bounded features}}: $\l \phi(s) \r \le 1,$ for all $s\in\S$.\\
\item The step sizes satisfy $\sum_{n}\gamma_n = \infty$, and $\sum_n\gamma_n^2 <\infty$.\\
\item \textit{\textbf{Combined step size and mixing assumption}}: There exists a non-negative function $B'(\cdot)$ such that: For all $s \in \S$ and $m\ge0$,
\begin{align*}  
&\sum\limits_{\tau=0}^\infty e^{ 3(1+\beta)\sum_{j = 1}^{\tau-1}\gamma_\tau}
   														\left\| \E( r(s_{\tau},\pi(s_{\tau}) )\phi(s_\tau)\mid s_0 = s) - \E_\Psi(r( s_{\tau},\pi(s_{\tau}) )\phi(s_\tau))\right\| \le B'(s), \\
   &\sum\limits_{\tau=0}^\infty e^{ 3(1+\beta)\sum_{j = 1}^{\tau-1}\gamma_\tau}
   													\left\| \E(\phi(s_\tau)\phi(s_{\tau+m})\tr \mid s_0 = s)- \E_\Psi(\phi(s_\tau)\phi(s_{\tau+m})\tr)\right\| \le B'(s),
\end{align*}
\end{enumerate}
\textbf{(A6')} \textit{\textbf{Uniform mixing bound}}: (A6) holds, and there exists a constant $B'$ that is an uniformly bound on $B(s), \forall s \in \S$.

%The above assumptions are similar in nature to those made in \citep{tsitsiklis1997analysis} for establishing asymptotic convergence of TD(0) with linear function approximators. 
In comparison to the assumptions in \citep{tsitsiklis1997analysis}, (A1), (A3), (A5) have exact counterparts in \citep{tsitsiklis1997analysis}, while (A2), (A4) and (A6) are simplified versions of the corresponding boundedness assumptions in \citep{tsitsiklis1997analysis}.

\begin{remark}\textbf{\textit{(Geometric ergodicity)}}
A Markov chain is mixing at a \textit{geometric rate} if 
\begin{align}
P(s_t=s\mid s_0) - \psi(s)| \le C \rho^{t}.
\label{eq:mixing-fast}
\end{align}
For finite state space settings, the above condition holds and hence (A6) is easily satisfied. Moreover, $B' =  \Theta\left(1/(1-(1-\rho)^{1-\epsilon}\right)$, for any $\epsilon > 0$. Here $\rho$ is an unknown quantity that relates to the second eigenvalue of the transition probability matrix. See Chapters 15 and 16 of \citep{meyn2009markov} for a detailed treatment of the subject matter.
\end{remark}

%\footnote{We assume an upper bound of $1$ for both rewards and features for the sake of simplicity and we believe our analysis can be extended (left for future work) to the following generalized variants:\\
%(\textbf{A2}') The rewards satisfy $\E_\Psi(r^2(s,\pi(s))) < \infty$, $\forall s \in \S$;\\
%(\textbf{A4}') The feature vector $\phi_k(s)$ for any $k=1,\ldots,d$ and $s\in\S$ satisfies $\E_\Psi(\phi_k^2(s)) < \infty$.}.

\subsection{Non-averaged case}
\label{sec:conc-bounds-td}

\begin{theorem}
\label{thm:td-rate}
Under \textbf{(A1)-(A6)}, choosing $\gamma_n= \frac{c_0c}{(c+n)}$, with $c_0< \mu(1-\beta)/(2(1+\beta)^2)$ and $c$ such that $\mu (1-\beta)c_0c >1$, we have, 
\begin{align*}
 \E \l \theta_n - \theta^* \r \le \dfrac{K_1(n)}{\sqrt{n+c}}.
\end{align*}
In addition, assuming \textbf{(A6')}, we have. for any $\delta >0$,
\begin{align*}
&\P\left( \l \theta_n - \theta^* \r \le \dfrac{K_2(n)}{\sqrt{n+c}}\right) \ge 1 - \delta, \text{ where}\\
&K_1(n):= 
	\left( \frac{c(\l\theta_0 - \theta^*\r+C)}{(n+c)^{\mu(1-\beta)c_0c-1}} 
					+  \frac{(1+\l\theta^*\r)c_0^2c^2 + C c_0 c}{\mu(1-\beta)c_0c-1}
					\right)^\frac{1}{2},\\
&K_2(n):= \frac{c_0cB'
					\left(2\left[2 + c_0c\right]
					\left[1 + \beta(3-\beta)\right]
					\ln(1/\delta)\right)^\frac{1}{2}}
					{ \left( \mu (1-\beta)c_0 c -1 \right)^\frac{1}{2} }
					+ K_1(n),\\
&C := 6 d B(s_0) \left(\frac{\l \theta_{0}\r + d + \l \theta^*\r}{1-\beta}\right)^2.
\end{align*}
\end{theorem}
\begin{proof}
 See Section \ref{sec:non-av-proof}.
\end{proof}
%%%%%%%%%%%%%%%%%%%%%%%%%%%%%%%%%%%%%%%%%%%%%%%%%%%%%%%%%%%%%%

\begin{remark}
$K_1(n)$ and $K_2(n)$ above are $O(1)$, i.e., they can be upper bounded by a constant. Thus, one can indeed get the optimal rate of convergence of the order $O\left(1/\sqrt{n}\right)$ with a step-size $\gamma_n= \frac{c}{(c+n)}$. However, this rate is contingent upon on the constant $c$ in the step-size being chosen correctly. This is problematic because the right choice of $c$ requires the knowledge of eigenvalue $\mu$ for expectation bound and and knowing $\mu$ would imply knowledge about the transition probability matrix of the underlying Markov chain. The latter information is unavailable in a typical RL setting. The next section derives bounds for the iterate averaged variant that overcomes this problematic step-size dependency.
%For finite state space settings, the mixing bound $B(s_0)$ can be shown to be a constant (see Section VII of \citep{tsitsiklis1997analysis}), but one that depends on the second eigenvalue of the transition probability matrix and the latter information is unavailable in a typical RL setting. 
\end{remark}
% and this problematic as knowing $B(s_0)$ would imply knowledge about the transition probability matrix, in particular the second eigenvalue, of the underlying Markov chain. Even in the case of a finite state space, it is easy to infer that $B(s_0)$ is a constant , but one still has to know this constant value explicitly to set $c$. 

%%%%%%%%%%%%%%%%%%%%%%%%%%%%%%%%%%%%%%%%%%%%%%%%%%%%%%%%%%%%%%
%%%%%%%%%%%%%%%%%%%%%%%%%%%%%%%%%%%%%%%%%%%%%%%%%%%%%%%%%%%%%%
%%%%%%%%%%%%%%%%%%%%%%%%%%%%%%%%%%%%%%%%%%%%%%%%%%%%%%%%%%%%%%
%%%%%%%%%%%%%%%%%%%%%%%%%%%%%%%%%%%%%%%%%%%%%%%%%%%%%%%%%%%%%%
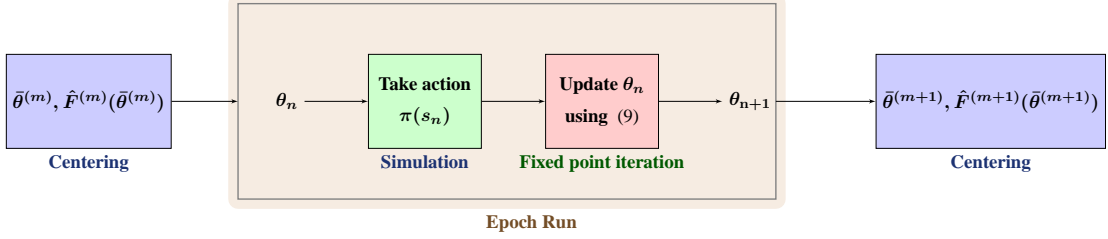
\begin{figure*}[t]
\centering
\tikzstyle{block} = [draw, fill=white, rectangle,
   minimum height=5em, minimum width=6em]
\tikzstyle{sum} = [draw, fill=white, circle, node distance=1cm]
\tikzstyle{input} = [coordinate]
\tikzstyle{output} = [coordinate]
\tikzstyle{pinstyle} = [pin edge={to-,thin,black}]
\scalebox{0.65}{\begin{tikzpicture}[auto, node distance=2cm,>=latex']
\node (theta) at (0,0) {$\boldsymbol{\theta_n}$};
\node [block, fill=green!20, right=1.3cm of theta,label=below:{\color{bleu2}\bf Simulation},align=center] (sample) {\textbf{Take action }\\[1ex] $\boldsymbol{\pi(s_n)}$}; 
\node [block, fill=red!20, right=1.3cm of sample,label=below:{\color{darkgreen}\bf Fixed point iteration},align=center] (update) {\textbf{Update} $\boldsymbol{\theta_{n}}$ \\[1ex]\textbf{using } \eqref{eq:ctd-update} };
\node [right=1.3cm of update] (end) {$\boldsymbol{\mathbf{\theta_{n+1}}}$};
\draw [->] (theta) --  (sample);
% \draw [->] (sample) -- node {$r(s_n,\pi(s_n))$} (update);
\draw [->] (sample) -- (update);
\draw [->] (update) -- (end);
\draw [color=gray,thick](-1,-2) rectangle (10,2);
\node at (5,-2.5) {\textbf{\color{brown!60!black}Epoch Run}};
\node [block, fill=blue!20,label=below:{\color{bleu2}\bf Centering}, left=2cm of theta,align=center] (avgbeg)  
{$\boldsymbol{\bar\theta^{(m)},\hat F^{(m)}(\bar\theta^{(m)})}$}; 
\node [block, fill=blue!20, right=2cm of end,label=below:{\color{bleu2}\bf Centering},align=center] (avgend)  
{$\boldsymbol{\bar\theta^{(m+1)},\hat F^{(m+1)}(\bar\theta^{(m+1)})}$}; 
\draw [->] (avgbeg) -- (-1,0);
\draw [->] (10,0) -- (avgend);

\begin{pgfonlayer}{background}
    \filldraw [line width=4mm,join=round,brown!15]
(-1,-2) rectangle (10,2);
  \end{pgfonlayer}
\end{tikzpicture}}
\caption{Illustration of centering principle in CTD algorithm.}
\label{fig:algorithm-flow-ctd}
\end{figure*}
%%%%%%%%%%%%%%%%%%%%%%%%%%%%%%%%%%%%%%%%%%%%%%%%%%%%%%%%%%%%%%

%%%%%%%%%%%%%%%%%%%%%%%%%%%%%%%%%%%%%%%%%%%%%%%%%%%%%%%%%%%%%%
\subsection{Iterate Averaging}
\label{sec:iterate-av}
The idea here is to employ larger step-sizes and combine it with averaging of the iterates, i.e., $\bar \theta_{n+1} := (\theta_1+\ldots+\theta_n)/n$. This principle was introduced independently by Ruppert \citep{ruppert1991stochastic} and Polyak \citep{polyak1992acceleration}, for accelerating stochastic approximation schemes. The following theorem establishes that iterate averaging results in the optimal rate of convergence without any step-size dependency:
\begin{theorem}
\label{thm:td-avg-rate}
Under \textbf{(A1)-(A6)}, choosing $\gamma_n= c_0\left(\frac{c}{c+n}\right)^\alpha$, with $\alpha \in (1/2,1)$ and $c \in (0,\infty)$, we have, for all $n>n_0:=(c\mu(1-\beta)/(2c_0(1+\beta)^2))^{-1/\alpha}$,
\begin{align*}
\E \l \bar\theta_n -  \theta^* \r
	\le \dfrac{K_1^{IA}(n)}{(n+c)^{\alpha/2}}.
\end{align*}
In addition, assuming \textbf{(A6)'}, we have, for any $\delta >0$,
\begin{align*}
\P\left( \l \bar\theta_n - \theta^* \r
	\le \dfrac{K_2^{IA}(n)}{(n+c)^{\alpha/2}}\right)
		\ge 1 - \delta, 
\end{align*}
where 
\begin{align*}
K_1^{IA}(n) :=& \dfrac{((1+dc_0c^{\alpha}(c+n_0)^{1-\alpha})
							e^{(1+\beta)c_0c^{\alpha}(c+n_0)^{1-\alpha})}+\|\theta^*\|+C)C''}
								{(n+c)^{(1-\alpha)/2}}\\
&		+\dfrac{n_0 \left[(1+dc_0c^{\alpha}(c+n_0)^{1-\alpha})
							e^{(1+\beta)c_0c^{\alpha}(c+n_0)^{1-\alpha})}
							+ \|\theta^*\|\right]}
							{n^{1 - \alpha/2}}\\
&	+  c^\alpha c_0\left[1+ \l\theta^*\r^\frac{1}{2}
																+\left(\frac{C}{c^\alpha c_0}\right)^\frac{1}{2}\right]
					\left(\frac{\mu(1-\beta)c_0c^\alpha}{1-\alpha}\right)^{-\frac{\alpha+2\alpha^2}{2(1-\alpha)}},\\
K_2^{IA}(n)
	:= & \frac{4\sqrt{(1 +  C') B'}}{\mu(1-\beta)}	\frac{C'''}{n^{(1 - \alpha)/2}}\;+\; K_1^{IA}(n-n_0),\\
C' :=& \left( 3^\alpha + \left[\frac{4\alpha}{\mu(1-\beta)c_0c^{\alpha}}
						+ \frac{2^\alpha}{\alpha}\right]^2\right)^{\frac{1}{2}}, 
 C'' := \sum\limits_{k = 1}^\infty k^{-2\alpha},
\text{ and }\\
C''' :=& \sum\limits_{k=1}^{\infty}
						e^{-\frac{\mu c^\alpha (1-\beta)c_0}{2(1-\alpha)}((n+c)^{1-\alpha}-((c+n_0)^{1-\alpha})}.
\end{align*}						
\end{theorem}
\begin{proof}
 See Section \ref{sec:it-av-proof}.
\end{proof}
\begin{remark}
The step-size exponent $\alpha$ can be chosen arbitrarily close to $1$, resulting in a convergence rate of the order $O\left( 1/\sqrt{n}\right)$. However although the constants $K_1^{IA}(n)$ and $K_2^{IA}(n)$ remain $O(1)$, there is a minor tradeoff here since a choice of $\alpha$ close to $1$ would result in their bounding constants blowing up. One cannot choose $c$ too large or too small for the same reasons. 
\end{remark}

%Thus, iterate averaging results in the optimal rate of convergence, while having no dependency for the choice of $c$ and this result is consistent with the asymptotic convergence rate results from \citep{kondathesis}. 
%%%%%%%%%%%%%%%%%%%%%%%%%%%%%%%%%%%%%%%%%%%%%%%%%%%%%%%%%%%%%%
%%%%%%%%%%%%%%%%%%%%%%%%%%%%%%%%%%%%%%%%%%%%%%%%%%%%%%%%%%%%%%
%%%%%%%%%%%%%%%%%%%%%%%%%%%%%%%%%%%%%%%%%%%%%%%%%%%%%%%%%%%%%%
%%%%%%%%%%%%%%%%%%%%%%%%%%%%%%%%%%%%%%%%%%%%%%%%%%%%%%%%%%%%%%
\section{TD(0) with Centering (CTD)}
\label{sec:fast-td}
CTD is a \textit{control variate} solution to reduce the variance of the updates of normal TD(0). This is achieved by adding a zero-mean, \textit{\textbf{centering}} term to the TD(0) update.
%\subsection{The Algorithm}

 Let $X_n = (s_n,s_{n+1})$. % Each epoch is of length $M$. \\
Then, the TD(0) algorithm can be seen to perform the following fixed-point iteration:
 \begin{align}
\theta_n = \theta_{n-1} + \gamma_{n} f_{X_n}(\theta_n).\label{eq:td-again}
\end{align}
where $f_{X_n}(\theta):= (r(s_n,\pi(s_n)) + \beta \theta\tr \phi(s_{n+1}) - \theta\tr \phi(s_n))\phi(s_{n})$.
The limit of \eqref{eq:td-again} is the solution, $\theta^*$, of $F(\theta) = 0$, where  $F(\theta):= \Pi T^\pi(\Phi \theta) - \Phi\theta$. 
The idea behind the CTD algorithm is to reduce the variance of the increments $f_{X_n}(\theta_n)$, in order that larger step sizes can be used. This is achieved by choosing an extra iterate $\bar\theta_n$, centred over the previous $\theta_n$, and using an increment approximating $f_{X_n}(\theta_n) - f_{X_n}(\bar\theta_n) + F(\bar\theta_n)$. The intuitive motivation for this choice is that when the CTD algorithm arrives close to $\theta^*$, the centering term alone ensures the updates become small, while with regular TD(0), one has to rely on a decaying step size to keep the iterates close to $\theta^*$.

%\paragraph{Difficulties in centering.}
The approach is inspired by the SVRG algorithm, proposed in \citep{johnson2013accelerating}, for a optimising a strongly-convex function.
However, the setting for TD(0) with function approximation that we have is considerably more complicated owing to the following reasons:\\
  \begin{inparaenum}[\bfseries(i)]
   \item Unlike \citep{johnson2013accelerating}, we are not optimising a function that is a finite-sum of smooth functions in a batch setting. Instead, we are estimating a value function which is an infinite (discounted) sum, with the individual functions making up the sum being made available in an online fashion (i.e. as new samples are generated from the simulation of the underlying MDP for policy $\pi$).\\
   \item The centering term in SVRG directly uses $F(\cdot)$, which in our case is a limit function that is neither directly accessible nor can be simulated for any given $\theta$.\\
   \item Obtaining the exponential convergence rate is also difficult owing to the fact that TD(0) does not initially see samples from the stationary distribution and there is an underlying mixing term that affects the rate.\\
   \item Finally, there are extra difficulties owing to the fact that we have a fixed point iteration, while the corresponding algorithm in \citep{johnson2013accelerating} is stochastic gradient descent (SGD). 
  \end{inparaenum}

The CTD algorithm that we propose overcomes the difficulties mentioned above and the overall scheme of this epoch-based algorithm is presented in Figure \ref{fig:algorithm-flow-ctd}. At the start of the $m^{th}$ epoch, a random iterate is picked from the previous epoch, i.e. $\bar\theta^{(m)} = \theta_{i_n}$, where $i_n$ is drawn uniformly at random in $\{(m-1)M,\ldots,mM\}$. Thereafter, for the epoch length $M$, CTD performs the following iteration: Set $\theta_{mM} = \bar\theta^{(m)}$ and for $n=mM,\ldots,(m+1)M-1$ update
 \begin{align}
\theta_{n+1} =& \Upsilon\bigg(\theta_{n} + \gamma \Big(f_{X_{i_n}}(\theta_n) - f_{X_{i_n}}(\bar\theta^{(m)}) 
 + \hat F^{(m)}(\bar\theta^{(m)}) \Big)\bigg),\label{eq:ctd-update}
\end{align}
where $\hat F^{(m)}(\theta) := M^{-1}\sum_{i=(m-1)M}^{mM} f_{X_i}(\theta)$ and $\Upsilon$ is a projection operator that ensures that the iterates stay within a $H$-ball. 
% \begin{remark}
 Unlike TD(0), one can choose a large (constant) stepsize $\gamma$ in \eqref{eq:ctd-update}. This choice in conjunction with iterate averaging via the choice of $\bar\theta^{(m)}$ results in an exponential convergence rate for CTD (see Remark \ref{remark:ctd-rate} below). 

% \end{remark}

\subsection{Finite time bound}
Theorem \ref{thm:ctd-bound} below presents a finite time bound in expectation for CTD under the following mixing assumption:\\
\textbf{(A6'')} There exists a non-negative function $B'(\cdot)$ such that: For all $s \in \S$ and $m\ge0$,
  \begin{align*}  
   &\sum\limits_{\tau=0}^\infty	\left\| \E( r(s_{\tau},\pi(s_{\tau}) )\phi(s_\tau)\mid s_0 = s)
- \E_\Psi(r( s_{\tau},\pi(s_{\tau}) )\phi(s_\tau))\right\| \le B'(s), \\
   &\sum\limits_{\tau=0}^\infty	\left\| \E[\phi(s_\tau)\phi(s_{\tau+m})\tr \mid s_0 = s]  																- \E_\Psi[\phi(s_\tau)\phi(s_{\tau+m})\tr]\right\| \le B'(s),
  \end{align*}
The above is weaker than assumption (A6) used earlier for regular TD(0), and this is facilitated by the fact that we project the CTD iterates onto a $H$-ball.
 \begin{theorem}
\label{thm:ctd-bound}
Assume  (A1)-(A4) and (A6'') and let $\theta^*$ denote the solution of $F(\theta)=0$.
Let the epoch length $M$ of the CTD algorithm \eqref{eq:ctd-update} be chosen such that $C_1<1$, where
\begin{align*}
C_1:=((2\mu\gamma M)^{-1} + \gamma d^2/2)/((1-\beta) - d^2\gamma/2))
\end{align*}
%Then, under (A1)-(A4) and (A6),  we have,\\
\textbf{(i) Geometrically ergodic chains:}
Here the Markov chain underlying policy $\pi$ mixes fast (see \eqref{eq:mixing-fast}) and we obtain\footnotemark
\begin{align}
 &\|\Phi (\bar\theta^{(m)} - \theta^*)\|_{\Psi}^2 
 \le C_1^m  \left( \|\Phi (\bar\theta^{(0)} - \theta^*)\|_{\Psi}^2\right)  + CM C_2H(5\gamma + 4)  \max\{C_1,\rho^M\}^{(m - 1)},\label{eq:fastmixing-ctd-rate}
\end{align}
where $C_2= \gamma/(M ((1-\beta) - d^2\gamma/2))$.
\footnotetext{For any $v\in\R^d$, we take $\| v \| _\Psi := \sqrt{v\tr\Psi v}$.}

\textbf{(ii) General Markov chains:}
\begin{align}
 &\|\Phi (\bar\theta^{(m)} - \theta^*)\|_{\Psi}^2 
 \le C_1^m\left( \|\Phi (\bar\theta^{(0)} - \theta^*)\|_{\Psi}^2\right) + C_2H(5\gamma + 4) \sum_{k = 1}^{m-1} C_1^{(m - 2) -k} B_{(k-1)M}^{kM}(s_0),\label{eq:ctd-rate}
\end{align}
where $B_{(k-1)M}^{kM}$ is an upper bound on the partial sums $\sum_{i=(k-1)M}^{kM}(\E(\phi(s_i)\mid s_0) - \E_{\Psi}(\phi(s_i)))$ and $\sum_{i=(k-1)M}^{kM}(\E(\phi(s_i)\phi(s_{i+l})\mid s_0) - \E_{\Psi}(\phi(s_i)\phi(s_{i+l})\tr))$, for $l = 0,1$.\\[1ex]
\end{theorem}
\begin{proof}
 See Section \ref{sec:ctd-proof}.
\end{proof}
For finite state space settings, we obtain exponential convergence rate \eqref{eq:fastmixing-ctd-rate} since they are geometrically ergodic, while for MDPs that do not mix exponentially fast, the second (mixing) term in \eqref{eq:ctd-rate} will dominate and decide the rate of the CTD algorithm.

\begin{remark}\label{remark:ctd-rate}
 Combining the result in \eqref{eq:fastmixing-ctd-rate} with the bound in statement ($4$) of Theorem $1$ in \citep{tsitsiklis1997analysis}, we obtain
\begin{align*}
 &\|\Phi \bar\theta^{(m)} - V^\pi\|_{\Psi} 
 \le \dfrac{1}{1-\beta} \left\| \Pi V^\pi - V^\pi \right\|_\Psi +  C_1^{m/2}\left( \|\Phi (\bar\theta^{(0)} - \theta^*)\|_{\Psi}\right) 
 + \sqrt{C C_2} \max\{C_1,\rho\}^{(m - 1)/2}. 
\end{align*}
The first term on the RHS above is an artifact of function approximation, while the second and third terms reflect the convergence rate of the CTD algorithm. 
\end{remark}

\begin{remark}
As a consequence of  the fact that $(\bar\theta^{(m)} - \theta^*)^T I (\bar\theta^{(m)} - \theta^*) \le \dfrac{1}{\mu}(\bar\theta^{(m)} - \theta^*)^T \Phi^T \Psi \Phi (\bar\theta^{(m)} - \theta^*)$, one can obtain the following bound on the parameter error for CTD:
\begin{align*}
 &\|\bar\theta^{(m)} - \theta^*\|_2 
 \le (1/\mu)\bigg( C_1^m\left( \|\Phi (\bar\theta^{(0)} - \theta^*)\|_{\Psi}^2\right) + C_2H(5\gamma + 4) \sum_{k = 1}^{m-1} C_1^{(m - 2) -k} B_{(k-1)M}^{kM}(s_0)\bigg).
\end{align*}
Comparing the above bound with those in Theorems \ref{thm:td-rate}--\ref{thm:td-avg-rate}, we can infer that CTD exhibits an exponential convergence rate of order $O(C_1^m)$, while TD(0) with/without averaging can converge only at a sublinear rate of order $O(n^{-1/2})$.
\end{remark}
%\begin{remark}
% To get the optimal rate for the CTD algorithm, we need to know the value of $\mu$ to set the step-size $\gamma$. However, we can get rid of this dependency by explicitly regularizing the problem. This would imply that we solve $(A+\mu I)\theta^*=b$ instead of \eqref{eq:td-system-of-equations}, where we choose $\mu$.  
%\end{remark}
%
%%%%%%%%%%%%%%%%%%%%%%%%%%%%%%%%%%%%%%%%%%%%%%%%%%%%%%%%%%%%%%
%%%%%%%%%%%%%%%%%%%%%%%%%%%%%%%%%%%%%%%%%%%%%%%%%%%%%%%%%%%%%%
%%%%%%%%%%%%%%%%%%%%%%%%%%%%%%%%%%%%%%%%%%%%%%%%%%%%%%%%%%%%%%
\section{Convergence proofs}
\label{sec:analysis}
%In this section we give the major details of proofs of the results, however full proofs of all results can be found in the appendices.
\subsection{Non-averaged case: Proof of Theorem \ref{thm:td-rate}}
\label{sec:non-av-proof}
We split the analysis in two, first considering the bound in high probability, and second the bound in expectation. Both bounds involve a martingale decomposition, the former of the centered error, and the latter of the iteration \eqref{eq:td-update-intro}. 
% Also, both analyses begin with a result for general choices of the step-size sequence, and then specialise these results to obtain the rates in Theorem \ref{thm:td-rate}.
\subsubsection{High probability bound}
We first state and prove a result bounding the error with high probability for general step-sizes:
\begin{proposition}\textbf{\textit{(High probability bound)}}
\label{thm:high-probability-bound-td}
Under (A1)-(A5) and (A6'),  we have,
\begin{align*}
P(  \l \theta_n - \theta^* \r - \E \l \theta_n - \theta^* \r \ge \epsilon ) 
\le e^{- \epsilon^2\left( 2 \sum_{i=1}^{n} L^2_i\right)^{-1}},
\end{align*}
where 
$L_i := \gamma_i  [ e^{ -\mu(1-\beta)\sum_{k = i}^{n} \gamma_k }
						 ( 1 + [ \gamma_i +\sum_{k = i}^{n-1}[\gamma_k - \gamma_{k+1} ]
														e^{\mu(1-\beta)\sum_{j = i}^{k+1} \gamma_j }]
														[1 + \beta(3-\beta)]B'
														)
 								]^{\frac{1}{2}}$.
\end{proposition}
\begin{proof}
Recall that $z_n:=\theta_n - \theta^*$. First, we  rewrite {$\l z_n \r^2 - E \l z_n \r^2$} as a telescoping sum of martingale differences:
\begin{align}
\l z_n \r - \E \l z_n \r
	= \sum\limits_{i=1}^{n} g_i - \E [ g_i \left| \F_{i-1} \right.] 
\label{eq:prob-equivalence}
\end{align}
where
$D_i := g_i - \E[ g_i \left| \F_{i-1} \right.]$, $g_i := \E [\l z_n\r \left| \theta_i, \F_{i-1}\right.]$, and $\F_i$ denotes the sigma algebra generated by the states of the underlying Markov chain, $\{\sigma_1,\dots,\sigma_i\}$. Note that $\theta_i$ is $\F_i$-measurable.

The above establishes that the centered error, $\l z_n\r - \E \l z_n \r$ can be written down as a sum of martingale differences with respect to the filtration $\{\F_i\}_{i = 0}^n$. 
The proof procedes by establishing that these martingale differences are Lipschitz functions of the random inovation at each time, with Lipschitz constants $L_i$.
This is the content of Lemma \ref{lemma:gi-lipschitz-random-lstd}, and it makes use of assumption (A3), (A4), and (A7).
This is the since the random innovation is, through assumptions (A2)-(A4), bounded, it is subgaussian, and we can invoke a standard concentration argument in Lemma \ref{lemma:conc-bound} to finish the bound.

%%%%%%%%%%%%%%%%%%%%%%%%%%%%%%%%%%%%%%%%%%%%%%%%%%%%%%%%%%%%%%%%%%%%%%%%%%%%%%%%%%%%%%%%
\begin{lemma}
  \label{lemma:gi-lipschitz-random-lstd}
Recall that $X_n = (s_n,s_{n+1})$ and $f_{X_n}(\theta):= \left(r(s_n,\pi(s_n)) + \beta \theta \tr \phi(s_{n+1}) - \tr \phi(s_n)\right)\phi(s_{n})$. Then, conditioned on $\F_{i-1}$, the functions $g_i$ are Lipschitz continuous in $f_{X_i}(\theta_{i-1})$, the random innovation at time $i$, with constants 
\begin{align*}
L_i := \gamma_i  \left[ e^{ -\mu(1-\beta)\sum_{k = i}^{n} \gamma_k }
						 \left( 1 + \left[ \gamma_i + 
														\sum_{k = i}^{n-1}
														\left[\gamma_k - \gamma_{k+1} \right]\right]
														\left[1 + \beta(3-\beta)\right]B'
														\right)
 								\right]^{\frac{1}{2}}.
 \end{align*}
 \end{lemma}
 \begin{proof}
 Let $\Theta_j^i(\theta)$ denote the mapping that returns the value of the iterate $\theta_j$ at instant $j$, given that $\theta_i = \theta$.
\begin{align*}
   \Theta_{j+1}^i(\theta) - \Theta_{j+1}^i(\theta') 
 	&=  \Theta_{j}^i(\theta) - \Theta_{j}^i(\theta')
 		- \gamma_{j} [ f_{X_{j}}(\Theta_{j}^i(\theta))
 											- f_{X_{j}}(\Theta_{j}^i(\theta')) ]\\
 	&= \Theta_{j}^i(\theta) - \Theta_{j}^i(\theta')
 		- \gamma_{j} [\phi(s_{j})\phi(s_{j})\tr - \beta\phi(s_{j})\phi(s_{{j+1}})\tr] 
 															(\Theta_{j}^i(\theta) - \Theta_{j}^i(\theta'))\\
	&=  [I - \gamma_{j} a_j] (\Theta_{j}^i(\theta) - \Theta_{j}^i(\theta'))
	= \left(\prod_{k = i}^{j} [I - \gamma_{k} a_k]  \right)
				(\theta - \theta'),
\end{align*}
where $a_j:= \phi(s_{j})\phi(s_{j})\tr - \beta\phi(s_{j})\phi(s_{{j+1}})\tr$.
So we have
\begin{align}
\E\left(\l\Theta_{n+1}^i(\theta) - \Theta_{n+1}^i(\theta')\r^2\mid \F_i \right)
 	= (\theta - \theta')\tr
 				\E\left[\left(\prod_{k = i}^{n}[I - \gamma_k a_k]  \right)\tr
 							\left(\prod_{k = i}^{n}[I - \gamma_k a_k]  \right)\mid \F_i \right]
				(\theta - \theta') \label{eq:lipschitz0}
\end{align}
and from the tower property of conditional expectations, it follows that
\begin{align}
&\E\left[\left(\prod_{k = i}^{n}[I - \gamma_k a_k]  \right)\tr
 							\left(\prod_{k = i}^{n}[I - \gamma_k a_k]  \right)\mid \F_i \right] \nonumber\\
	=& \E\left[\left(\prod_{k = i}^{n-1}[I - \gamma_k a_k]  \right)\tr
 							\left(I - 2\gamma_{n}
 								\E\left(  a_{n} -  \frac{\gamma_{n}}{2}a_{n}\tr a_{n} \mid \F_n \right) \right)
 							\left(\prod_{k = i}^{n-1}[I - \gamma_k a_k]  \right)\mid \F_i \right] \nonumber\\
	=& \E\left[\left(\prod_{k = i}^{n-1}[I - \gamma_k a_k]  \right)\tr
 							\left(I - 2\gamma_{n}
 								\E_{\Psi}\left( a_{n} -  \frac{\gamma_{n}}{2}a_{n}\tr a_{n} \right) \right)
 							\left(\prod_{k = i}^{n-1}[I - \gamma_k a_k]  \right)\mid \F_i \right] \nonumber\\
		&+ \E\left[\left(\prod_{k = i}^{n-1}[I - \gamma_k a_k]  \right)\tr
 							2\gamma_{n} \E\left( \epsilon_n \mid \F_n \right) 
 							\left(\prod_{k = i}^{n-1}[I - \gamma_k a_k]  \right)\mid \F_i \right] \label{eq:lipschitz1}
\end{align}
where
\begin{align*}
\epsilon_{n}': = \E\left[(a_{n} -  \gamma_{n}a_{n}\tr a_{n}/2) \mid \F_n \right] 
								- \E_{\Psi}\left[ (a_{n} -  \gamma_{n}a_{n}\tr a_{n}/2) \right]
\end{align*}

To deal with the term in the second last line of \eqref{eq:lipschitz1}, let $\Delta$ be the diagonal matrix with entries $\Delta_{i,i} = \Phi_{i,1:d}\Phi_{i,1:d}\tr$.
Then we find that, for any vector $\theta$,
\begin{align}
\theta\tr&\E_\Psi\left[a_{j+1} - \frac{\gamma_{j+1}}{2}a_{j+1}\tr a_{j+1}\right]\theta
	=\theta\tr\Phi\tr \left(I-\beta \Psi P - \frac{\gamma_{j+1}}{2}\left(\Delta - \beta P \tr \left(2 I - \beta\Delta\right) \Psi P\right)\right)\Phi\theta\nonumber\\
	&=  \theta\tr\Phi\tr \left(I-\beta \Psi\Pi P - \frac{\gamma_{j+1}}{2}\left(\Delta - \beta P \tr \Pi\tr \left(2 I - \beta\Delta\right)\Psi \Pi P\right)\right)\Phi\theta \label{eq:l1}\\
    	&\ge \vml \Phi \theta\vmr^2 - \beta \vml \Phi \theta\vmr -  \frac{\gamma_{j+1}}{2}\theta\tr\Phi\tr\left(\Delta - \beta P \tr \Pi\tr \left(2 I - \beta\Delta\right)\Psi \Pi P\right)\Phi\theta \label{eq:l22}\\
    	&\ge\vml \Phi \theta\vmr^2 - \beta \vml \Phi \theta\vmr^2 - \frac{\gamma_{j+1}}{2}\vml \Phi \theta\vmr^2 + \frac{\gamma_{j+1}}{2}\theta\tr \Phi\tr \beta P \tr \Pi\tr \left(2 I - \beta\Delta\right)\Psi \Pi P\Phi\theta \label{eq:l23}\\
	&\ge\vml \Phi \theta\vmr^2 - \beta \vml \Phi \theta\vmr^2 - \frac{\gamma_{j+1}}{2}\vml \Phi \theta\vmr^2 + \frac{\gamma_{j+1}}{2}\beta(2-\beta)\vml \Pi P\Phi\theta \vmr^2\label{eq:l2}\\
	&\ge  \mu \left(1-\beta - \frac{\gamma_{j+1}}{2}\right) \l \theta\r^2, \label{eq:l3}
\end{align}
where \eqref{eq:l1} follows from the fact that $\theta\tr\Phi\tr \Psi (I-\Pi)x =0$ since $\Pi$ is a projection, \eqref{eq:l22} by an application of Cauchy-Schwarz inequality and from the non-expansiveness property of $\Pi$ and $P$. \eqref{eq:l23} and \eqref{eq:l2} follow from the fact that, by (A4), the matrix $\Delta - I$ is positive semi-definite. The final inequality \eqref{eq:l3} follows from (A3).

For the term on the last line of \eqref{eq:lipschitz1} we notice, that by boundedness of the features,
\begin{align*}
\l\prod_{k = i}^{n-1}[I - \gamma_k a_k]  \r
	\le \prod_{k = i}^{n-1} \left( 1 + \gamma_k (1+\beta) \right).
\end{align*}
Using this observation and \eqref{eq:l3}, we can expand \eqref{eq:lipschitz0} with \eqref{eq:lipschitz1} to obtain
\begin{align}
\E&\left(\l\Theta_{n+1}^i(\theta) - \Theta_{n+1}^i(\theta')\r^2\mid \F_i \right)\nonumber\\
 	&\le x_n (\theta - \theta')
 								\E\left[
 								\left(\prod_{k = i}^{n-1}[I - \gamma_k a_k]  \right)\tr
 								\left(\prod_{k = i}^{n-1}[I - \gamma_k a_k]  \right)\mid \F_i
 								\right]
 								(\theta - \theta') \\
 		&\qquad\qquad\qquad\qquad\qquad
 			+ 2\gamma_n
 			\left(\prod_{j = i}^{n-1} (1 + \gamma_j (1+\beta) )\right)^2
 			\l\E\left( \epsilon_n \mid \F_i \right) \r 
 															  \l \theta - \theta'\r^2 \\
% 	&\le x_n (\theta - \theta')
% 								\E\left[
% 								\left(\prod_{k = i}^{n-1}[I - \gamma_k a_k]  \right)\tr
% 								\left(\prod_{k = i}^{n-1}[I - \gamma_k a_k]  \right)\mid \F_i
% 								\right]
% 								(\theta - \theta') \\
% 		&\qquad\qquad\qquad\qquad\qquad
% 			+ 2\gamma_n
% 			\left(2H\right)^2
% 			\l\E\left( \epsilon_n \mid \F_i \right) \r 
% 															  \l \theta - \theta'\r^2 \\
 	&\le \left[ \left(\prod_{k = i}^{n} x_k\right) 
 		+ 2\sum_{k = i}^{n}\gamma_k\left(\prod_{j = k+1}^{n} x_j\right)
 			\left(\prod_{j = i}^{k-1} (1 + \gamma_j (1+\beta) )\right)^2
 			\l\E\left( \epsilon_k \mid \F_i \right) \r \right]
 															  \l \theta - \theta'\r^2 \nonumber\\
 	&\le \Bigg[ e^{ -2\sum_{k = i}^{n} \gamma_k\mu \left(1-\beta - \frac{\gamma_{k}}{2}\right) }
 		+ 2\sum_{k = i}^{n}\gamma_k
 			e^{-2\sum_{j = k+1}^{n}  \gamma_j\mu \left(1-\beta - \frac{\gamma_{j}}{2}\right)
 					+ 2\sum_{j = i}^{k-1} \gamma_j (1+\beta) ) }
 			\l\E\left( \epsilon_k \mid \F_i \right) \r \Bigg]
 															  \l \theta - \theta'\r^2\\
 	&\le e^{ -2\sum_{k = i}^{n} \gamma_k\mu \left(1-\beta - \frac{\gamma_{k}}{2}\right) }
 		\left[ 1+ 2\sum_{k = i}^{n}\gamma_k
 			e^{2\sum_{j = i}^{k}  \gamma_j(\mu 
 																	\left(1-\beta - \frac{\gamma_{j}}{2}\right)
 																	+ (1+\beta) ) }
 			\l\E\left( \epsilon_k \mid \F_i \right) \r \right]
 															  \l \theta - \theta'\r^2
\end{align}
where $x_n := 1 - 2\gamma_n\mu \left(1-\beta - \frac{\gamma_{n}}{2}\right)$.
Now, using discrete integration by parts,
\begin{align*}
\sum_{k = i}^{n}&\gamma_k
 			e^{2\sum_{j = i}^{k}  \gamma_j(\mu 
 																	\left(1-\beta - \frac{\gamma_{j}}{2}\right)
 																	+ (1+\beta) ) }
 			\l\E\left( \epsilon_k \mid \F_i \right) \r\\
	=&  \gamma_i
					\sum_{k = i}^{n} e^{2\sum_{j = i}^{k}  \gamma_j(\mu 
 																	\left(1-\beta - \frac{\gamma_{j}}{2}\right)
 																	+ (1+\beta) ) }
 																	\l\E\left( \epsilon_k \mid \F_i \right) \r\\
		&+ \sum_{k = i}^{n-1}
		\left[\gamma_{k+1} - \gamma_k \right]
		 \sum_{j = k+1}^{n} e^{ 2\sum_{l = i}^{j} \gamma_l(\mu 
 																	\left(1-\beta - \frac{\gamma_{l}}{2}\right)
 																	+ (1+\beta) ) }
 																	\l\E\left( \epsilon_k \mid \F_i \right) \r\\
	\le& \left[\gamma_i 
		+ \sum_{k = i}^{n-1} \left[\gamma_{k+1} - \gamma_k \right]\right]
																		 \left[1 + \beta(3-\beta)\right]B'(s_i)
\end{align*}
where for the inequality we have used (A6) with the following bound:
\begin{align*}
\l \E(\epsilon_k\mid \F_i )\r \
\le& (1-\gamma_k/2) \l \E(\beta \phi(s_k)\phi(s_{k})\tr \mid \F_i) - \E_\Psi(\beta \phi(s_k)\phi(s_{k})\tr)\r\\
& + \beta \l \E(\beta \phi(s_k)\phi(s_{k+1})\tr \mid \F_i) - \E_\Psi(\beta \phi(s_k)\phi(s_{k+1})\tr)\r\\
& + \beta(2-\beta) \l \E(\phi(s_{k+1})\phi(s_{k+1})\tr \mid \F_i) - \E_\Psi(\phi(s_{k+1})\phi(s_{k+1})\tr )\r.
\end{align*}
From the foregoing, we have
\begin{align}
&\E\left[ \l \Theta^i_n\left(\theta\right) - \Theta^i_n\left(\theta'\right) \r\right]\\
	\le& \left[ e^{ -\mu(1-\beta)\sum_{k = i}^{n} \gamma_k }
						 \left( 1 + \left[ \gamma_i + 
														\sum_{k = i}^{n-1}
														\left[\gamma_k - \gamma_{k+1} \right]\right]
														\left[1 + \beta(3-\beta)\right]B'(s_i)
														\right)
 								\right]^{\frac{1}{2}}
 						\l \theta - \theta'\r \label{eq:hpbna}
\end{align}
Finally, we have
\begin{align*}
 &\left| \E\left[ \l \theta_n - \theta^* \r \mid \F_{i-1},\theta_i = \theta \right]
 			- \E\left[ \l \theta_n - \theta^* \r \mid \F_{i-1},\theta_i = \theta' \right] \right|_2
	\le  \E\left[ \l \Theta^i_n\left(\theta\right) - \Theta^i_n\left(\theta'\right) \r\right]\\
	&\quad\le \left[ e^{ -\mu(1-\beta)\sum_{k = i}^{n} \gamma_k }
						 \left( 1 + \left[ \gamma_i + 
														\sum_{k = i}^{n-1}
														\left[\gamma_k - \gamma_{k+1} \right]\right]
														\left[1 + \beta(3-\beta)\right]B'(s_i)
														\right)
 								\right]^{\frac{1}{2}}
 						\gamma_i \l f - f'\r\\
	&\quad\le  L_i \l f - f'\r.
\end{align*}
where $f = \theta - \theta_{i-1}$ and $f' = \theta' - \theta_{i-1}$.

\end{proof}

In the following lemma, we invoke a standard martingale concentration bound using the $L_i$-Lipschitz property of the $g_i$ functions and the assumption (A3).
\begin{lemma}\label{lemma:conc-bound}
Under the conditions of Theorem \ref{thm:high-probability-bound-td}, we have
\begin{align}
  \label{eq:thm1-concbound}
 P(   \l z_n \r - \E \l z_n \r \ge \epsilon )
\le \e(-\lambda \epsilon) \e\bigg(\dfrac{\alpha \lambda^2}{2} \sum\limits_{i=1}^{n} L^2_i \bigg).
\end{align}
\end{lemma}
\begin{proof}Note that
  \begin{align*}
 P(   \l z_n \r - \E \l z_n \r \ge \epsilon )& =  P\left(  \sum\limits_{i=1}^{n} D_i  \ge \epsilon \right)
  \le \e(-\lambda \epsilon) \E\left(\e\bigg(\lambda \sum\limits_{i=1}^{n} D_i\bigg)\right) \\
  &=  \e(-\lambda \epsilon) \E\left(\e\bigg(\lambda \sum\limits_{i=1}^{n-1} D_i\bigg) \E\bigg(\e(\lambda D_n) \left| \F_{n-1}\right.\bigg)\right).
 \end{align*}
The first equality above follows from \eqref{eq:prob-equivalence}, while the inequality follows from Markov inequality.
Now for any bounded random variable $f$, and $L$-Lipschitz function g we have
$$\E \left( \e(\lambda g(f))\right) \le \e\left( \lambda^2 L^2 / 2 \right).$$
Note that each $f_i(\theta_{i-1})$ is a bounded random variable by (A2) and (A4), and, conditioned on $\F_{i-1}$, $g_i$ is Lipschitz in $f_i(\theta_{i-1})$ with constant $L_i$ (Lemma \ref{lemma:gi-lipschitz-random-lstd}). So we obtain
\begin{align*}
 \E\left(\e(\lambda D_n) \left| \F_{n-1}\right.\right) \le \e\left(\dfrac{\lambda^2 L^2_n}{2}\right),
\end{align*}
and so
  \begin{align*}
    P(   \l z_n \r - \E \l z_n \r \ge \epsilon ) \le \e(-\lambda \epsilon) \e\bigg(\dfrac{\alpha \lambda^2}{2} \sum\limits_{i=1}^{n} L^2_i \bigg).
\end{align*}
\end{proof}
The proof of Proposition \ref{thm:high-probability-bound-td} follows by optimizing over $\lambda$ in \eqref{eq:thm1-concbound}.
\end{proof}

%%%%%%%%%%%%%%%%%%%%%%%%%%%%%%%%%%%%%%%%%%%%%%%%%%%%%%%%%%%%%%%%%%%%%%%%%%%%%%%%%%%%%%%%%%%%%%%%%%%%%%%%%%%%%%%%%%%%%%%%%%%%%%%%%%%%%%%%%%%%%%%%%
%%%%%%%%%%%%%%%%%%%%%%%%%%%%%%%%%%%%%%%%%%%%%%%%%%%%%%%%%%%%%%%%%%%%%%%%%%%%%%%%%%%%%%%%%%%%%%%%%%%%%%%%%%%%%%%%%%%%%%%%%%%%%%%%%%%%%%%%%%%%%%%%%
%%%%%%%%%%%%%%%%%%%%%%%%%%%%%%%%%%%%%%%%%%%%%%%%%%%%%%%%%%%%%%%%%%%%%%%%%%%%%%%%%%%%%%%%%%%%%%%%%%%%%%%%%%%%%%%%%%%%%%%%%%%%%%%%%%%%%%%%%%%%%%%%%
%%%%%%%%%%%%%%%%%%%%%%%%%%%%%%%%%%%%%%%%%%%%%%%%%%%%%%%%%%%%%%%%%%%%%%%%%%%%%%%%%%%%%%%%%%%%%%%%%%%%%%%%%%%%%%%%%%%%%%%%%%%%%%%%%%%%%%%%%%%%%%%%%
\subsubsection{Bound in expectation}
Now we state and prove a result bounding the expected error for general step-size sequences:
\begin{proposition}\textbf{\textit{(Bound in expectation)}}
\label{thm:expectation-bound-td}
Under (A1)-(A6) and assuming that $\gamma_n\le\mu(1-\beta)/(2(1+\beta)^2)$ for all $n$, we have,
 \begin{align}
  \E(  \l \theta_{n+1} - \theta^* \r \left\| s_0\right.) 
	\le& \Bigg[ e^{- \mu(1-\beta)\sum_{k=1}^{n}\gamma_k}  (\l z_0\r+ C)+ (1+ \l\theta^*\r) \sum_{k=1}^{n}\gamma_k^2
													e^{- \mu(1-\beta)\sum_{j=k}^{n}\gamma_j}\nonumber\\		  	 
&\quad+  C\sum_{k=1}^{n-1}(\gamma_{k+1} - \gamma_k)
											e^{- \mu(1-\beta)\sum_{j=k+1}^{n}\gamma_j}
																											 \Bigg]^\frac{1}{2}, \label{eq:thm4}
\end{align}
where $C = 2(2+\beta)(d+4)B(s_0)
		 \left(\frac{\l \theta_{0}\r + d + \l \theta^*\r}{1-\beta}\right)^2$.
% where $\Gamma_k:=\sum_{i=1}^k \gamma_i$ and $\l \theta_n \r \le H, \forall n$.
% where
% $h(k) := 2\beta\ln k$.
\end{proposition}
\begin{proof}
 The update rule \eqref{eq:td-update-intro} can be re-written as follows:
\begin{align}
\label{eq:td-equiv}
 z_{n+1} =& \theta_n - \theta^*
 						  + \gamma_n [ \E_\Psi(f_{X_n}(\theta_n))
 									+ \E(f_{X_n}(\theta_n)\mid \F_n) - \E_\Psi(f_{X_n}(\theta_n))
 									+ f_{X_n}(\theta_n) - \E(f_{X_n}(\theta_n)\mid \F_n)],\nonumber\\
 						=& z_n
 						  + \gamma_n [ \E_\Psi(f_{X_n}(\theta_n) - f_{X_n}(\theta^*)) 
 									+ \E(f_{X_n}(\theta_n) - f_{X_n}(\theta^*)\mid \F_n)
 													 			- \E_\Psi(f_{X_n}(\theta_n) - f_{X_n}(\theta^*))\nonumber\\
 									&\qquad\qquad + f_{X_n}(\theta_n) - f_{X_n}(\theta^*)
 													 			- \E(f_{X_n}(\theta_n) - f_{X_n}(\theta^*)\mid \F_n)
 									+ f_{X_n}(\theta^*)]
\end{align}

We notice that
\begin{align}
\E_\Psi(f_{X_n}(\theta_n) - f_{X_n}(\theta^*))
 	= & \E_\Psi\left( \beta \theta_n \tr \phi(s_{n+1})- \beta {\theta^*}\tr\phi(s_{n+1})
 								- (\theta_n \tr \phi(s_{n}) - {\theta^*}\tr\phi(s_{n}))\right)\nonumber\\
 	=& \E_\Psi\left( (\theta_n - \theta^*)\tr
 									[\beta \phi(s_{n+1}) - \phi(s_n)] \phi(s_n) \right)\nonumber\\
 	=& \E_\Psi\left( \phi(s_n)[\beta \phi(s_{n+1})\tr - \phi(s_n)\tr]
 																		(\theta_n - \theta^*)  \right)\nonumber\\
 	=& -A z_n,\label{eq:term1}
\end{align}
where $A:=\Phi\tr \Psi(I - \beta P )\Phi$ (here $P=P_\pi$ denotes the one-step transition probability matrix of the underlying Markov chain induced under a stationary policy $\pi$). Noticing also that for any vector $x$
\begin{align*}
x\tr\Phi\tr\Psi P\Phi x
 = \langle x\tr\Phi\tr,  P\Phi x\rangle_\Psi^2
 \le \| x\tr\Phi\tr \|_\Psi \|  P\Phi x\|_\Psi
 =\| x\tr\Phi\tr \|_\Psi^2,
\end{align*}
where we have used Lemma 1 in \cite{tsitsiklis1997analysis} for the final equality. So we deduce that $A - (1-\beta)\mu I$ is positive definite by (A3).

Furthermore, letting
\begin{gather*}
a_n :=  \beta \phi(s_{n}) \phi(s_{n+1}) \tr - \phi(s_n) \phi(s_n)\tr
\text{, }
\epsilon_n :=  \E(a_n\mid \F_n)- \E_\Psi(a_n)
\text{, and }
\Delta M_n :=  a_n- \E(a_n\mid \F_n)),				
\end{gather*}
then, using \eqref{eq:term1}, we can rewrite \eqref{eq:td-equiv} as
\begin{align*}
z_{n+1}
	= \left[I - \gamma_n ( A + \epsilon_n + \Delta M_n )\right] z_n
 									+ \gamma_n \epsilon_n'.
\end{align*}
where $\epsilon_n' = f_{X_n}(\theta^*)-\E_{\Psi}\left(f_{X_n}(\theta^*)\right)$. Notice that $\Delta M_n$ is a martingale difference sequence with respect to the filtration $\F = \{ \F_n \}_{n\ge 1}$, and $\E(\epsilon_n\mid s_0)$ is mixing-error term. Taking the expectation of the square, and expanding, we have
\begin{align}
\E&\left(\l z_{n+1}\r^2\mid\F_n\right)
	= \E\left(\l\left[ I - \gamma_n ( A + \epsilon_n + \Delta M_n ) \right] z_n\r^2\mid \F_n \right)\\
	&\qquad\qquad\qquad\qquad+ 2\gamma_n  \E\left(f_{X_n}(\theta^*)\tr
																\left[I - \gamma_n ( A + \epsilon_n + \Delta M_n )\right] z_n\mid\F_n\right)
 												+ \gamma_n^2 (1+(1+\beta)\l\theta^*\r)^2\nonumber\\
	&= [z_n\tr[ I - 2\gamma_n ( A + \E\left(\epsilon_n\mid \F_n \right))
												+ \gamma_n^2\E\left(( A + \epsilon_n + \Delta M_n )^2\mid\F_n\right)] z_n]\\
	&\qquad				+ 2\gamma_n \E\left(\left(f_{X_n}(\theta^*)
																				-\E_{\Psi}\left(f_{X_n}(\theta^*)\right)\right)\tr
																\left[I - \gamma_n ( A + \epsilon_n + \Delta M_n )\right] z_n\mid\F_n\right)
 								+ \gamma_n^2 (1+(1+\beta)\l\theta^*\r)^2\nonumber\\
 	&\le [1 - 2\gamma_n (\mu(1-\beta) - 2\gamma_n(1+\beta)^2)]\l z_n\r^2\label{eq:td-recursion}\\
 	&\qquad				+ \gamma_n z_n\tr\E\left(\epsilon_n\mid\F_n\right)z_n
 			+ 2\gamma_n \E\left(\left(\epsilon_n'\right)\tr
																\left[I - \gamma_n ( A + \epsilon_n + \Delta M_n )\right]
																		z_n\mid \F_n\right)\nonumber
 			+ \gamma_n^2 (1+(1+\beta)\l\theta^*\r)^2
\end{align}
we have used that
$$\l A+ \epsilon_n + \Delta M_n \r= \l\beta\phi(s_{n})\phi(s_{n+1})\tr  - \phi(s_{n})\phi(s_{n})\tr)\r \le (1+\beta).$$

Before unrolling \eqref{eq:td-recursion} using the tower property of expectations, we bound the terms involving $\epsilon_n$ and $\epsilon_n'$.
We notice first that, given a fixed orthonomal basis of vectors $\{ e_i \}_{i=1}^d$. and a the associated coordinate mapping $\alpha_i(\cdot)$ for which $\phi = \sum_{i=1}^d\alpha_i(\phi)e_i$ holds true,
\begin{align*}
\theta_{n+1}
	&= \theta_{n} + \gamma_{n} \left(r(s_n,\pi(s_n)) + \beta \theta_{n}\tr \phi(s_{n+1}) 
																							- \theta_{n}\tr \phi(s_n)\right) \phi(s_{n})\\
	&= (1+\gamma_{n}a_n)\theta_n+ \gamma_nr(s_n,\pi(s_n))\phi(s_n)\\
	&= \prod_{k = 1}^n(1+\gamma_{k}a_n) \theta_{0}
			+ \sum_{k = 1}^n \gamma_k\left(\prod_{j = k}^n(1+\gamma_{j}a_j)\right)
																				r(s_k,\pi(s_k))\phi(s_k)\\
	&= \left(\prod_{k = 1}^n(1+\gamma_{k}a_n) \right)\theta_{0}
			+ \sum_{i=1}^d\left[\sum_{k = 1}^n \alpha_i(\phi(s_k))r(s_k,\pi(s_k))
							\gamma_k\left(\prod_{j = k}^n(1+\gamma_{j}a_j)\right)\right]e_i.
\end{align*}
Now, if $A$ is a possibly random matrix such that $\l A\r\le C$, and $\theta$ is a fixed vector, then
\begin{align*}
\E\left(\theta\tr A\tr\E(\epsilon_n\mid\F_n)A\theta\mid s_0\right)
	\le \E\left(C^2\theta\tr\E(\epsilon_n\mid\F_n)\theta\mid s_0\right)
	\le C^2 \theta\tr \E\left(\epsilon_n\mid s_0\right) \theta.
\end{align*}
Furthermore, for any norm $\|\cdot\|$, $\| \sum_{i=1}^d x_i\|^2\le d\sum_{i=1}^d\| x_i\|^2$. From these observations we can deduce that
\begin{align*}
\E&\left(z_{n+1}\E(\epsilon_n\mid\F_n)\tr z_{n+1}\mid s_0\right)\\
	&\le (d+2)\left(e^{2(1 + \beta)\sum_{k = 1}^n\gamma_{k}}\l \theta_{0}\r^2
			+ \left(\sum_{k = 1}^n \gamma_k e^{(1 + \beta)\sum_{j = k}^n\gamma_{j}}\right)^2
			+ \l\theta^*\r^2\right)
			\l\E(\epsilon_n\mid s_0)\r\\
	&\le (d+2)\left(e^{2(1 + \beta)\sum_{k = 1}^n\gamma_{k}}\l \theta_{0}\r
			+ e^{2(1 + \beta)\sum_{k = 1}^n\gamma_{k}}
				\left(\sum_{k = 1}^n \gamma_k e^{-(1 + \beta)\sum_{j = 1}^{k-1}\gamma_{j}}\right)^2
			+ \l\theta^*\r^2\right)
			\l\E(\epsilon_n\mid s_0)\r\\
	&\le (d+2)\frac{\l \theta_{0}\r^2 + 1 + \l \theta^*\r^2}{\left( 1-\beta \right)^2}
				e^{2(1 + \beta)\sum_{k = 1}^n\gamma_{k}}
				\l\E(\epsilon_n\mid s_0)\r.
\end{align*}
Similarly, using Cauchy-Schwarz,
\begin{align*}
\E&\left(\E\left(\left(\epsilon_n'\right)\tr
	\left[I - \gamma_n ( A + \epsilon_n + \Delta M_n )\right] z_n\mid \F_n\right)\mid s_0\right)\\
	&= \E\left(\left(\epsilon_n'\right)\tr
							\left[I - \gamma_n ( A + \epsilon_n + \Delta M_n )\right]
							\prod_{k = 1}^n(1+\gamma_{k}a_n)\mid s_0\right)\theta_0\\
	&\qquad + \sum_{i=1}^d\E\left(\left(\epsilon_n'\right)\tr
							\left[I - \gamma_n ( A + \epsilon_n + \Delta M_n )\right]
							\sum_{k = 1}^n \gamma_k \alpha_i(\phi(s_k))r(s_j,\pi(s_j))
									\left(\prod_{j = k}^n (1+\gamma_{j}a_j)\right)
									\mid s_0\right)e_i\\
	&\qquad  - 	\left[I - \gamma_n ( A + \epsilon_n + \Delta M_n )\right]
							\E\left(\left(\epsilon_n'\right)\tr\mid s_0\right)\theta^*\\
	&\le 2e^{(1 + \beta)\sum_{k = 1}^n\gamma_{k}}
									\l \E\left(\left(\epsilon_n'\right)\mid s_0\right) \r \l \theta_{0} \r\\
	&\qquad 	+ 2d\l \E \left(\left(\epsilon_n'\right)\mid s_0\right)\r
									 e^{(1 + \beta)\sum_{k = 1}^n\gamma_{k}}
									\sum_{k = 1}^n \gamma_k e^{-(1 + \beta)\sum_{j = 1}^{k-1}\gamma_{j}}
					+ 2\l \E \left(\left(\epsilon_n'\right)\mid s_0\right)\r\l\theta^*\r\\
	&\le2\frac{\l \theta_{0}\r + d + \l \theta^*\r}{1-\beta}
					e^{(1 + \beta)\sum_{k = 1}^n\gamma_{k}}
					\l \E \left(\left(\epsilon_n'\right)\mid s_0\right)\r
\end{align*}
So we can unroll \eqref{eq:td-recursion} using the tower property of conditional expectations to obtain:
\begin{align*}
\E&(\l z_{n+1}\r^2\mid s_0)
	\le \tpi_n z_0 + (1+ (1+\beta)\l\theta^*\r)^2 \sum_{k=1}^{n}\gamma_k^2 \tpi_n\tpi_k^{-1} \\
		 &	\qquad + (d+2)\frac{\l \theta_{0}\r^2 + d + \l \theta^*\r^2}{\left( 1-\beta \right)^2}
											\sum_{k=1}^{n}\gamma_k\tpi_n\tpi_k^{-1}
											e^{2(1 + \beta)\sum_{j  = 1}^k\gamma_{j}}
																										\l\E(\epsilon_k\mid s_0)\r\\
		 &\qquad	 + 4\frac{\l \theta_{0}\r + d + \l \theta^*\r}{1-\beta}
											\sum_{k=1}^{n}\gamma_k\tpi_n\tpi_k^{-1}
											e^{(1 + \beta)\sum_{j  = 1}^k\gamma_{j}}
																										\l\E(\epsilon_k'\mid s_0)\r\\
	&\le  e^{-\sum_{k=1}^{n}2\gamma_k (\mu(1-\beta) - 2\gamma_k(1+\beta)^2)} z_0
			+ (1+ (1+\beta)\l\theta^*\r)^2 \sum_{k=1}^{n}\gamma_k
													e^{-\sum_{j=k}^{n}\gamma_j (\mu(1-\beta) - 2\gamma_j(1+\beta)^2)} \\
		 &	\qquad + (d+2)\frac{\l \theta_{0}\r^2 + d + \l \theta^*\r^2}{\left( 1-\beta \right)^2}
											\sum_{k=1}^{n}\gamma_k
											e^{-\sum_{j=k}^{n}\gamma_j (\mu(1-\beta) - 2\gamma_j(1+\beta)^2)}
											e^{2(1 + \beta)\sum_{j  = 1}^k\gamma_{j}}
																										\l\E(\epsilon_k\mid s_0)\r \\
		 &	\qquad	 + 4\frac{\l \theta_{0}\r + d + \l \theta^*\r}{1-\beta}
											\sum_{k=1}^{n}\gamma_k
											e^{-\sum_{j=k}^{n}\gamma_j (\mu(1-\beta) - 2\gamma_j(1+\beta)^2)}
											e^{(1 + \beta)\sum_{j  = 1}^k\gamma_{j}}
																										\l\E(\epsilon_k'\mid s_0)\r
\end{align*}
 where $\tpi_n:=\prod_{k=1}^{n}\left(1 - 2\gamma_n (\mu(1-\beta) - 2\gamma_n(1+\beta)^2)\right)$.

Using that $\gamma_n< \mu(1-\beta)/(2(1+\beta)^2)$ for all $n$, we have that 
\begin{align*}
\E(\|z_{n+1}\|\mid s_0)
	\le& \left[\E(\l z_{n+1}\r^2\mid s_0)\right]^\frac{1}{2}\\
	\le& \Bigg[ e^{- \mu(1-\beta)\sum_{k=1}^{n}\gamma_k}  z_0
			+ (1+ (1+\beta)\l\theta^*\r)^2 \sum_{k=1}^{n}\gamma_k^2
													e^{- \mu(1-\beta)\sum_{j=k}^{n}\gamma_j} \\
		 &	 + (d+2)\frac{\l \theta_{0}\r^2 + d + \l \theta^*\r^2}{\left( 1-\beta \right)^2}
											\sum_{k=1}^{n}\gamma_k
											e^{- \mu(1-\beta)\sum_{j=k}^{n}\gamma_j} 
											e^{2(1 + \beta)\sum_{j  = 1}^k\gamma_{j}}
																										\l\E(\epsilon_k\mid s_0)\r\\
		 &	 + 4\frac{\l \theta_{0}\r + d + \l \theta^*\r}{1-\beta}
											\sum_{k=1}^{n}\gamma_k
											e^{- \mu(1-\beta)\sum_{j=k}^{n}\gamma_j} 
											e^{(1 + \beta)\sum_{j  = 1}^k\gamma_{j}}
																										\l\E(\epsilon_k'\mid s_0)\r
							\Bigg]^\frac{1}{2}
\end{align*}
Using discrete integration by parts, we have that
\begin{align*}
\sum_{k=1}^{n}&
			e^{- \mu(1-\beta)\sum_{j=k}^{n}\gamma_j} 
			e^{2(1 + \beta)\sum_{j  = 1}^k\gamma_{j}}
							\l\E(\epsilon_k\mid s_0)\r\\
	\le& e^{- \mu(1-\beta)\sum_{k=1}^{n}\gamma_j} \sum_{j=1}^{n}
											e^{2(1 + \beta)\sum_{j  = 1}^k\gamma_{j}}
																										\l\E(\epsilon_k\mid s_0)\r\\
				&+ \sum_{k=1}^{n-1}(\gamma_{k+1} - \gamma_k)
									e^{- \mu(1-\beta)\sum_{j=k+1}^{n}\gamma_j}
								\sum_{j=1}^{k}e^{2(1 + \beta)\sum_{l  = 1}^j\gamma_{l}}
																										\l\E(\epsilon_j\mid s_0)\r\\
	\le& (1+\beta)B(s_0)
			\left[e^{- \mu(1-\beta)\sum_{k=1}^{n}\gamma_j}
				 	+ \sum_{k=1}^{n-1}(\gamma_{k+1} - \gamma_k)
									e^{- \mu(1-\beta)\sum_{j=k+1}^{n}\gamma_j}\right].
\end{align*}
Treating the mixing term involving $\epsilon_k'$ similarly, we find that:
\begin{align*}
\E(&\l z_{n+1}\r\mid s_0)
	\le \left[\E(\l z_{n+1}\r^2\mid s_0)\right]^\frac{1}{2}\\
	\le& \Bigg[ e^{- \mu(1-\beta)\sum_{k=1}^{n}\gamma_k}  \l z_0\r
			+ (1+ \l\theta^*\r) \sum_{k=1}^{n}\gamma_k^2
													e^{- \mu(1-\beta)\sum_{j=k}^{n}\gamma_j} \\
		 &	 +\left[ (d+2)\frac{\l \theta_{0}\r^2 + d + \l \theta^*\r^2}{\left( 1-\beta \right)^2}
											(1+\beta)B(s_0)
		    + 4\frac{\l \theta_{0}\r + d + \l \theta^*\r}{1-\beta}
											(2+\beta)B(s_0) \right]\\
		 & \qquad\qquad\qquad\qquad\qquad\qquad
		 						\times\left[e^{- \mu(1-\beta)\sum_{k=1}^{n}\gamma_j}
				 							+ \sum_{k=1}^{n-1}(\gamma_{k+1} - \gamma_k)
											e^{- \mu(1-\beta)\sum_{j=k+1}^{n}\gamma_j}\right]
																											 \Bigg]^\frac{1}{2}\\
	\le& \Bigg[ e^{- \mu(1-\beta)\sum_{k=1}^{n}\gamma_k}  \l z_0\r
			+ (1+ \l\theta^*\r) \sum_{k=1}^{n}\gamma_k^2
													e^{- \mu(1-\beta)\sum_{j=k}^{n}\gamma_j} \\
		 &	 + 2(2+\beta)(d+4)B(s_0)
		 \left(\frac{\l \theta_{0}\r + d + \l \theta^*\r}{1-\beta}\right)^2 
											\left[e^{- \mu(1-\beta)\sum_{k=1}^{n}\gamma_j}
				 							+ \sum_{k=1}^{n-1}(\gamma_{k+1} - \gamma_k)
											e^{- \mu(1-\beta)\sum_{j=k+1}^{n}\gamma_j}\right]
																											 \Bigg]^\frac{1}{2}\\
	\le& \Bigg[ e^{- \mu(1-\beta)\sum_{k=1}^{n}\gamma_k}  (\l z_0\r+ C)\\
		 & \qquad\qquad\qquad + (1+ \l\theta^*\r) \sum_{k=1}^{n}\gamma_k^2
													e^{- \mu(1-\beta)\sum_{j=k}^{n}\gamma_j}
		  	 +  C\sum_{k=1}^{n-1}(\gamma_{k+1} - \gamma_k)
											e^{- \mu(1-\beta)\sum_{j=k+1}^{n}\gamma_j}
																											 \Bigg]^\frac{1}{2}
\end{align*}
where $C = 2(2+\beta)(d+4)B(s_0)
		 \left(\frac{\l \theta_{0}\r + d + \l \theta^*\r}{1-\beta}\right)^2$.
\end{proof}

%%%%%%%%%%%%%%%%%%%%%%%%%%%%%%%%%%%%%%%%%%%%%%%%%%%%%%%%%%%%%%%%%%%%%%%%%%%%%%%%%%%%%%%%%%%%%%%%%%%%%%%%%%%%%%%%%%%%%%%%%%%%%%%%%%%%%%%%%%%%%%%%%
%%%%%%%%%%%%%%%%%%%%%%%%%%%%%%%%%%%%%%%%%%%%%%%%%%%%%%%%%%%%%%%%%%%%%%%%%%%%%%%%%%%%%%%%%%%%%%%%%%%%%%%%%%%%%%%%%%%%%%%%%%%%%%%%%%%%%%%%%%%%%%%%%
%%%%%%%%%%%%%%%%%%%%%%%%%%%%%%%%%%%%%%%%%%%%%%%%%%%%%%%%%%%%%%%%%%%%%%%%%%%%%%%%%%%%%%%%%%%%%%%%%%%%%%%%%%%%%%%%%%%%%%%%%%%%%%%%%%%%%%%%%%%%%%%%%
%%%%%%%%%%%%%%%%%%%%%%%%%%%%%%%%%%%%%%%%%%%%%%%%%%%%%%%%%%%%%%%%%%%%%%%%%%%%%%%%%%%%%%%%%%%%%%%%%%%%%%%%%%%%%%%%%%%%%%%%%%%%%%%%%%%%%%%%%%%%%%%%%

\subsubsection{Derivation of rates}

\begin{proof}\textbf{\textit{(High probability bound in Theorem \ref{thm:td-rate})}}\ \\
Note that when $\gamma_n = \frac{c_0 c}{(c+n)}$,
\begin{align*}
\sum_{i=1}^{n} L_i^2
 \le& \sum_{i=1}^{n}\frac{c_0^2c^2}{(c+i)^2}
 									\Bigg[e^{-\mu (1-\beta)c_0c \sum\limits_{j=i}^{n}\frac{1}{(c+j)}}
 									\left[1 + \frac{c_0 c}{c+i}
 											+ \sum_{k = i}^{n-1} \left(\frac{c_0c}{c+k}- \frac{c_0c}{c+k+1}\right)
 														\right]
 									\left[1 + \beta(3-\beta)\right]B'\Bigg]\\
 \le& \left[1 + \beta(3-\beta)\right]B'
 									\sum_{i=1}^{n}\frac{c_0^2c^2}{(c+i)^2}
 									\left(\frac{c+i}{c+n}\right)^{\mu(1-\beta)c_0c}
 									\left[2 + \sum_{k = i}^{n-1} \frac{c_0c}{(c+k)(c+k+1)}\right]\\
 \le& \frac{\left[2 + c_0c\right]\left[1 + \beta(3-\beta)\right]B'}{(c+n)^{\mu(1-\beta)c_0c}}
 									\sum_{i=1}^{n}\frac{c_0^2c^2}{(c+i)^{2 - \mu(1-\beta)c_0c}}
\end{align*}
We now find three regimes for the rate of convergence, based on the choice of $c$:\\
\begin{inparaenum}[\bfseries(i)]
\item $\sum_{i=1}^{n} L_i^2 = O\left((n+c)^{\mu (1-\beta)^2c }\right)$ when $\mu (1-\beta)c_0c \in(0,1)$, \\
\item $\sum_{i=1}^{n} L_i^2 = O\left(n^{-1}(\ln n)^2\right)$ when $\mu (1-\beta)c_0c =1$, and\\
\item $\sum_{i=1}^{n} L_i^2
					\le \frac{\left[2 + c_0c\right]\left[1 + \beta(3-\beta)\right]c_0^2c^2B'}
									{\mu (1-\beta)c_0 c -1 }(n+c)^{-1}$ when $\mu (1-\beta)c_0c\in(1,\infty)$.\\
\end{inparaenum}
(We have used comparisons with integrals to bound the summations.)
Setting $c$ so that we are in regime (iii), the high probability bound from Proposition \ref{thm:high-probability-bound-td} gives 
\begin{align}
\label{eq:prob-bound-1byn}
\P(   \l \theta_n - \theta^* \r - \E \l \theta_n - \theta^*\r \ge \epsilon ) \le \e\left(- \dfrac{\epsilon^2 (n+c)}{2 K_{\mu, c, \beta}}\right)
\end{align}
where $K_{\mu,c,\beta}
						:=\frac{\left[2 + c_0c\right]\left[1 + \beta(3-\beta)\right]c_0^2c^2B'}
									{\mu (1-\beta)c_0 c -1 }(n+c)^{-1}$.
\end{proof}

\begin{proof}\textbf{\textit{(Expectation bound in Theorem \ref{thm:td-rate})}} \ \\
For the expectation bound, the rate derivation involves bounding each term on the RHS in \eqref{eq:thm4} after choosing step-sizes $\gamma_n= \frac{c_0c}{(c+n)}$.
Supposing that $c$ is again chosen so that $(1-\beta)c_0\mu c> 1$ we have:
\begin{align*}
&\sum_{k=1}^{n}\gamma_k^2 e^{- \mu(1-\beta)\Gamma_k^n}
	\le \sum_{k=1}^{n}\frac{c_0^2c}{(c+k)^2}
										e^{- \mu(1-\beta)c_0c\sum_{i=k}^n\frac{1}{c+i}}\\
	&\le \frac{c_0^2c^2}{(c+n)^(\mu(1-\beta)c_0c-1)}
										\sum_{k=1}^{n}\frac{c_0^2c}{(c+k)^{2 - \mu(1-\beta)c_0c}}
	\le \frac{c_0^2c^2}{\mu(1-\beta)c_0c-1}\frac{1}{c+n}
\end{align*}
and similarly
\begin{align*}
\sum_{k=1}^{n-1}(\gamma_{k+1} - \gamma_k)
											e^{- \mu(1-\beta)\Gamma_{k+1}^n}
	\le \frac{2c_0c}{\mu(1-\beta)c_0c-1}\frac{1}{c+n}
\end{align*}
where we have again compared the sums with integrals. Also
\begin{align*}
e^{-(1-\beta)\mu\Gamma_1^n}
	\le \left(\frac{c}{n+c}\right)^{(1-\beta)\mu c_0 c} \le \left(\frac{c}{n+c}\right)^{\frac{1}{2}}.
\end{align*}
So from Proposition \ref{thm:expectation-bound-td}
\begin{align}
\label{eq:expectation-bound-1byn}
\E \l \theta_n - \theta^* \r
	\le \left( \frac{c(\l\theta_0 - \theta^*\r+C)}{(n+c)^{\mu(1-\beta)c_0c-1}} 
					+  \frac{(1+\l\theta^*\r)c_0^2c^2 + C c_0 c}{\mu(1-\beta)c_0c-1} \right)
		 (c+n)^{-\frac{1}{2}},
\end{align}
and the result in Theorem \ref{thm:td-rate} follows.
\end{proof}

%In order to obtain the rates and constants presented in Theorem \ref{thm:td-rate}, we specialize Theorems \ref{thm:expectation-bound-td} and \ref{thm:high-probability-bound-td} for the choice of step-size sequence, $\gamma_n= \frac{c_0c}{(c+n)}$, with $c_0< \mu(1-\beta)/(2(1+\beta)^2)$ and $c$ such that $\mu (1-\beta)c_0c \in(0,1)$.

\subsection{Iterate Averaging: Proof of Theorem \ref{thm:td-avg-rate}}
\label{sec:it-av-proof}
In order to prove the results in Theorem \ref{thm:td-avg-rate} we again consider the case of a general step sequence. Recall that $\bar \theta_{n+1} := (\theta_1+\ldots+\theta_n)/n$ and let $z_n = \bar \theta_{n+1} - \theta^*$.
We directly give a bound on the error in high probability for the averaged iterates (the bound in expectation can be obtained directly from the bound in Theorem \ref{thm:expectation-bound-td}):
\begin{proposition}
\label{thm:high-prob-bound-td-avg}
Suppose that $\forall n>n_0$, $\gamma_n\le\mu(1-\beta)/(2(1+\beta)^2)$. Then, under (A1)-(A5) and (A6'), and we have, $\forall\epsilon \ge 0$ and $\forall n>n_0$,
\begin{align*}
&P(   \l z_n \r - \E \l z_n \r \ge \epsilon ) \le e^{- \epsilon^2\left( 2 \sum_{i=1}^{n} L^2_i\right)^{-1} }, 
\end{align*}
where
$L_i := \frac{\gamma_i}{n}
				\bigg( 1 +\sum_{l=i+1}^{n-1}
								\Big[ e^{ -\mu(1-\beta)\sum_{k = i}^{l}\gamma_k}
						 					  ( 1 + [ \gamma_i + 
											  \sum_{k = i}^{l} [\gamma_k - \gamma_{k+1} ]]
											  [1 + \beta(3-\beta)]B' ) \Big]^{\frac{1}{2}}
 								\bigg)$.
\end{proposition}
\begin{proof}
The overall schema of this  proof is similar to that used for proving Proposition \ref{thm:high-probability-bound-td}, which is used for establishing a high probability bound of non-averaged TD(0).

We first decompose the centered error $ \l z_n \r - \E \l z_n \r$ into a sum of martingale differences as follows:
\begin{align}
 \l z_n \r - \E \l z_n \r
 	= \sum\limits_{k=1}^{n} D_k, \label{eq:prob-equivalence-averaged}
\end{align}
where $D_k := g_k - \E [ g_k \left| \F_{k-1} \right.]$ and $g_k := \sum_{k=1}^{l} \E [\l z_n\r \left| \F_k \right.]$.

We need to prove that the functions $g_k$ are Lipschitz continuous in the random innovation at time $k$ with the new constants $L_k$.
Let $ \zeta^1_k$ is the value of the averaged iterate $\bar \theta_{k-1}$ at instant $k$, $\zeta^2_k$ is the value of the iterate $\theta_k$ at instant $k$,
and let $\bar\Theta_j^k(\zeta)$ denote the mapping that returns the value of the averaged iterate at instant $j$, $\bar\theta_j$, given that $\bar\theta_{k-1} = \zeta^1$ and $\theta_k = \zeta^2$. Then, we have
\begin{align}
&\E \l \bar\Theta_{n}^k\left(\zeta\right) - \bar\Theta_{n}^k\left(\zeta'\right) \r
	\le \E\l \frac{k-1}{n}\left(\zeta^1 - {\zeta^1}'\right)
		+\frac{1}{n} \sum_{l = k}^n\left(\Theta^k_l\left(\zeta^2\right) - \Theta^k_l\left({\zeta^2}'\right)\right) \r
						\nonumber \\
	&\le  \dfrac{k-1}{n} \l \zeta^1 - \zeta'^1\r  \nonumber \\
	&\qquad+ \dfrac{1}{n} \sum\limits_{l=k}^{n}
												 \Big[ e^{ -\mu(1-\beta)\sum_{j = k}^{l}\gamma_j}
						 					  ( 1 + [ \gamma_k + 
											  \sum_{j =k}^{l} [\gamma_j - \gamma_{j+1} ]]
											  [1 + \beta(3-\beta)]B' ) \Big]^{\frac{1}{2}}
																					\l \zeta^2 - \zeta'^2\r. \label{eq:av1}
\end{align}
In the above, we have used \eqref{eq:hpbna} derived in the proof of Proposition \ref{thm:high-probability-bound-td} to bound $\l \Theta^k_l\left(\zeta^2\right) - \Theta^k_l\left({\zeta^2}'\right)\r$.
The rest of the proof amounts to showing that $g_k$ (and also $D_k$) is $L_k$-Lipschitz in the random innovation at time $k$ and this follows in a similar manner as the proof of Proposition \ref{thm:high-probability-bound-td}.
\end{proof}

\subsubsection{Derivation of rates: High Probability Bound in Theorem \ref{thm:td-avg-rate}}  
\label{sec:proof-lemma-it-av}
For the rate of the bound in high probability, one has to again separately bound the influence of the first $n_0$ steps, and then use the expectation bound together with \ref{thm:high-prob-bound-td-avg}. 
\begin{proof}\textbf{\textit{(High probability bound in Theorem \ref{thm:td-avg-rate})}}
We perform the calculation:
\begin{align*}
&\sum_{i = n_0}^n L_i^2
	= \sum_{i = n_0}^n
		\left[\frac{\gamma_i}{n}
				\left( 1 +\sum_{l=i+1}^{n-1}
								\left[ e^{ -\mu(1-\beta)\Gamma_i^n}
						 					  \left( 1 + \left[ \gamma_i + 
											  \sum_{k = i}^{n-1} \left[\gamma_k - \gamma_{k+1} \right]\right]
											  \left[1 + \beta(3-\beta)\right]B' \right) \right]^{\frac{1}{2}}
 								\right)\right]^2\\
	&\le \sum_{i = 1}^n
			\left[\frac{\gamma_i}{n}
				\left( 1 + \left( 1 + \left[ 1 +  C' \right] \left[1 + \beta(3-\beta)\right]B' \right)
								\sum_{l=i+1}^{n-1} e^{ -\frac{\mu(1-\beta)}{2}\Gamma_i^n}
 										\right)\right]^2\\
	&= \frac{1}{n^2}
		 \sum_{i = 1}^n\left[ c_0\left(\frac{c}{c+i}\right)^\alpha
		 						\left( 1 + \left( 1 + \left[ 1 +  C' \right] \left[1 + \beta(3-\beta)\right]B' \right)
								\sum_{l=i+1}^{n-1} e^{ -\frac{\mu(1-\beta)}{2}\Gamma_i^n}
 										\right)\right]^2\\
	&\le \left[ 1 +  C' \right] \left[1 + \beta(3-\beta)\right]B'
			\left(\frac{2c_0}{\mu(1-\beta)c_0n}\right)^2
			\sum_{i = 1}^n\bigg[\left(\frac{c+i+2}{c+i}\right)^\alpha+\\
 	&\qquad\qquad\qquad\qquad\qquad\qquad
 												\frac{1}{(c+i)^\alpha}\sum_{l=i}^{n-1}
 												e^{-\frac{\mu(1-\beta)c_0}{2}
 																	 c^\alpha\frac{ ((c+l)^{1-\alpha}- (c+i)^{1-\alpha})}{1-\alpha}}		
 												.((c+l+2)^\alpha-(c+l+1)^\alpha)\Bigg]^2\\
	&\le \left[ 1 +  C' \right] \left[1 + \beta(3-\beta)\right]B' 
			\left(\frac{2}{\mu(1-\beta)n}\right)^2
				\left\{ 3^\alpha + \left[\frac{4\alpha}{\mu(1-\beta)c_0c^{\alpha}}
						+ \frac{2^\alpha}{\alpha}\right]^2\right\}n
\end{align*}
where $C' := \sum_{k = 1}^\infty k^{-2\alpha}$.
In the second equality we have substituted $\gamma_i = (1-\beta)\left(\frac{c}{c+n}\right)^{\alpha}$. For the second inequality we have used discrete integration by parts (see page 15 in \cite{fathi2013transport}, display (2.2), for details). For the last inequality we have noted, as in  page 15 in \cite{fathi2013transport}, that
\begin{align}
\sum_{l=i+1}^{n-1}&
				e^{-\frac{\mu(1-\beta)c_0}{2}
					c^{\alpha}(1-\beta)^2
					((c+l)^{1-\alpha} - (c+i)^{1-\alpha})/(1- \alpha)}
				\left((c+l+2)^\alpha - (c+l+1)^\alpha\right) \nonumber\\
	\le& \frac{1}{1-\alpha}
				e^{\frac{\mu(1-\beta)c_0}{2}
					c^{\alpha}(1-\beta)^2 (c+i)^{1-\alpha}/(1-\alpha)}
					\int_{(c+i+1)^{1-\alpha}}^{(c+n)^{1-\alpha}}
														e^{\frac{\mu(1-\beta)c_0}{2}c^{\alpha}(1-\beta)^2 l/(1-\alpha)}
														l^{\frac{2\alpha-1}{1-\alpha}}dl. \label{eq:it-av-1}
\end{align}
Now, by taking the derivative and setting it to zero, we find that
$$l\mapsto
		e^{\frac{\mu(1-\beta)c_0}{2}c(1-\beta) l/(1-\alpha)}
		l^{\frac{2\alpha}{1-\alpha}}$$
 is decreasing on $[4\alpha/(\mu(1-\beta)c_0) c^{\alpha}(1-\beta),\infty)$, and so we deduce that \eqref{eq:it-av-1}$\le (c+i+1)^\alpha/\alpha$  when $c+i\ge4\alpha/\left(\mu(1-\beta)c_0\right)$. When $c+i<4\alpha/\left(\mu(1-\beta)c_0\right)$ we use that the summand is bounded by $1$.
\end{proof}

\subsubsection{Derivation of rates: Expectation bound in Theorem \ref{thm:td-avg-rate}}
\label{sec:proof-lemma-it-av-expectation}
To bound the expected error we average the errors of the non-averaged iterates. However, this averaging this is not straightforward as the bound in Theorem \ref{thm:expectation-bound-td} holds only if $n > n_0$ (which ensures that $\gamma_n$ is sufficiently small). Note that $n_0$ can be easily derived from the specific form of the step sequece. Hence we analyse the initial phase ($n< n_0$) and later phase $n\ge n_0$ separately as follows:
\begin{align*}
\E\l z_n\r
	\le& \frac{\sum_{k=1}^{n_0} \E\l\theta_k - \theta^*\r}{n}
			+ \frac{\sum_{k=n_0 +1}^n \E\l \theta_k - \theta^*\r}{n}
\end{align*}
The last term on RHS above is bounded using Theorem \ref{thm:td-rate}, while the first term is bounded by unrolling TD(0) recursion for the first $n_0$ steps and bounding the individual terms that arise using (A6). 
\begin{proof}\textbf{\textit{(Expectation bound in Theorem \ref{thm:td-avg-rate})}}
Let $n>n_0$. Then, with $\gamma_n = c_0\left(c/(c+n)\right)^{-\alpha}$, from the statement of Proposition \ref{thm:high-probability-bound-td}, we have
\begin{align*}
\E&\l \theta_n - \theta^* \r
	\le e^{-\frac{\mu c^\alpha (1-\beta)c_0}{2(1-\alpha)}((n+c)^{1-\alpha}-(c+n_0)^{1-\alpha})}
					(\l \theta_0 - \theta^*\r + C) \\
		&\quad +\left( 1+ \l\theta^*\r\right)^\frac{1}{2}
										\left( \sum_{k = n_0}^n c_0^2
															\left(\frac{c}{k+c+1}\right)^{2\alpha}
															2e^{-\frac{\mu(1-\beta)c_0c^\alpha}{1-\alpha}
															((n+c)^{1-\alpha} - (k+1+c)^{1-\alpha}}
															\right)^{\frac{1}{2}}\\
		&\quad +C^\frac{1}{2}
										\left( \sum_{k = n_0}^n c_0
															\left(\left(\frac{c}{k+c}\right)^{\alpha}
																	- \left(\frac{c}{k+c+1}\right)^{\alpha}\right)
															2e^{-\frac{\mu(1-\beta)c_0c^\alpha}{1-\alpha}
															((n+c)^{1-\alpha} - (k+1+c)^{1-\alpha}}
															\right)^{\frac{1}{2}}\\
	\le& e^{-\frac{\mu c^\alpha (1-\beta)c_0}{2(1-\alpha)}((n+c)^{1-\alpha}-(c+n_0)^{1-\alpha})}
		\Bigg[(\l \theta_0 - \theta^*\r+C)\\
					&\quad +\left[\left( 1+ \l\theta^*\r\right)^\frac{1}{2} c^\alpha c_0
											+(Cc_0c^{\alpha})^\frac{1}{2}\right]
									\left\{\int_{0}^{n}x^{-2\alpha}
										2e^{\frac{\mu(1-\beta)c_0c^\alpha}{1-\alpha} x^{1-\alpha}}dx
										\right\}^{\frac{1}{2}}
										\Bigg]\\
	\le& e^{-\frac{\mu c^\alpha (1-\beta)c_0}{2(1-\alpha)}((n+c)^{1-\alpha}-(c+n_0)^{1-\alpha})}
		\Bigg[(\l \theta_0 - \theta^*\r+C)\\
					&\quad +\left[\left( 1+ \l\theta^*\r\right)^\frac{1}{2} c^\alpha c_0
											+(Cc_0c^{\alpha})^\frac{1}{e}\right]
									\left\{\left(\frac{\mu(1-\beta)c_0c^\alpha}{1-\alpha}\right)^{-2\alpha}
									\int_{0}^{\left(\frac{\mu(1-\beta)c_0c^\alpha}{1-\alpha}\right)^{1/(1-\alpha)}n}
												y^{-2\alpha}2 e^{y^{1-\alpha}}dy
												\right\}^{\frac{1}{2}}
												\Bigg]\\
	\le& e^{-\frac{\mu c^\alpha (1-\beta)c_0}{2(1-\alpha)}((n+c)^{1-\alpha}-(c+n_0)^{1-\alpha})}
		\Bigg[(\l \theta_0 - \theta^*\r+C)\\
					&\quad +\left[\left( 1+ \l\theta^*\r\right)^\frac{1}{2} c^\alpha c_0
											+(Cc_0c^{\alpha})^\frac{1}{2}\right]
									\Bigg\{\left(\frac{\mu(1-\beta)c_0c^\alpha}{1-\alpha}\right)^{-2\alpha}\\
	&\quad\qquad\qquad\qquad\qquad
									.\int_{0}^{\left(\frac{\mu(1-\beta)c_0c^\alpha}{1-\alpha}\right)^{1/(1-\alpha)}n}
												((1-\alpha)y^{-2\alpha} - \alpha y^{-(1+\alpha)})
												2e^{y^{1-\alpha}}dy
													\Bigg\}^{\frac{1}{2}}
													\Bigg]\\
	\le& e^{-\frac{\mu c^\alpha (1-\beta)c_0}{2(1-\alpha)}((n+c)^{1-\alpha}-(c+n_0)^{1-\alpha})}
								(\l \theta_0 - \theta^*\r+C)\\
	&\quad\qquad\qquad\qquad\qquad
									+ \left[\left( 1+ \l\theta^*\r\right)^\frac{1}{2} c^\alpha c_0
											+(Cc_0c^{\alpha})^\frac{1}{2}\right]
										\left(\frac{\mu(1-\beta)c_0c^\alpha}{1-\alpha}\right)^{-\alpha\frac{1+2\alpha}{2(1-\alpha)}}
										(n+c)^{-\frac{\alpha}{2}}
\end{align*}
So, recalling that
\begin{align*}
\theta_{n_0} - \theta^*
	= \left(\prod_{k = 1}^{n_0}(1+\gamma_{k}a_n) \right)\theta_{0}
			+ \sum_{i=1}^d\left[\sum_{k = 1}^{n_0} \alpha_i(\phi(s_k))r(s_k,\pi(s_k))
							\gamma_k\left(\prod_{j = k}^n(1+\gamma_{j}a_j)\right)\right]e_i - \theta^*
\end{align*}
we deduce that
\begin{align*}
&\E\l \bar\theta_n - \theta^*\r\\
	\le& \frac{\sum_{k=1}^{n_0} \E\l\theta_k - \theta^*\r}{n}
			+ \frac{\sum_{k=n_0 +1}^n \E\l \theta_k - \theta^*\r}{n}\\
	\le& \frac{n_0 \left[(1+dc_0c^{\alpha}(c+n_0)^{1-\alpha})
							e^{2(1+\beta)c_0c^{\alpha}(c+n_0)^{1-\alpha})}
							+\|\theta^*\|\right]}{n}
			+ \frac{\sum_{k=n_0 +1}^n \E\l \theta_k - \theta^*\r}{n}\\
	\le&\frac{n_0 \left[(1+dc_0c^{\alpha}(c+n_0)^{1-\alpha})
							e^{2(1+\beta)c_0c^{\alpha}(c+n_0)^{1-\alpha})}
							+\|\theta^*\| \right]}{n}\\
			&+ \frac{(1+dc_0c^{\alpha}(c+n_0)^{1-\alpha})
							e^{2(1+\beta)c_0c^{\alpha}(c+n_0)^{1-\alpha})}+\|\theta^*\|+C}{n}
				\sum_{k=1}^{\infty}
						e^{-\frac{\mu c^\alpha (1-\beta)c_0}{2(1-\alpha)}((k+c)^{1-\alpha}-(c+n_0)^{1-\alpha})}\\
			&+  \left[\left( 1+ \l\theta^*\r\right)^\frac{1}{2} c^\alpha c_0
											+(Cc_0c^{\alpha})^\frac{1}{2}\right]
				\left(\frac{\mu(1-\beta)c_0c^\alpha}{1-\alpha}\right)^{-\alpha\frac{1+2\alpha}{2(1-\alpha)}}
										(n+c)^{-\frac{\alpha}{2}}
\end{align*}
\end{proof}

%%%%%%%%%%%%%%%%
%%%%%%%%%%%%%%%%%%%%%%%%%%%%%%%%%%%%%%%%%%%%%%%%%%%%%%%%%%%%%%
%%%%%%%%%%%%%%%%%%%%%%%%%%%%%%%%%%%%%%%%%%%%%%%%%%%%%%%%%%%%%%

\subsection{TD(0) with centering: Proof of Theorem \ref{thm:ctd-bound}}
\label{sec:ctd-proof}
\begin{proof}
The proof proceeds along the following steps:

\paragraph{Step 1: One-step expansion of the recursion \eqref{eq:ctd-update}.}
\begin{align}
&\|\theta_{n+1} - \theta^*\|_2^2\\
	=& \|\Upsilon\left[(\theta_{n} - \theta^*) + \gamma(f_{X_{i_n}}(\theta_n) - f_{X_{i_n}}(\bar\theta^{(m)}) + \E(f_{X_{i_n}}(\bar\theta^{(m)})\mid \F_n)\right]\|_2^2\nonumber\\
	\le& \|(\theta_{n} - \theta^*) + \gamma(f_{X_{i_n}}(\theta_n) - f_{X_{i_n}}(\bar\theta^{(m)}) + \E(f_{X_{i_n}}(\bar\theta^{(m)})\mid \F_n)\|_2^2\nonumber\\
 \le&\quad \|\theta_{n} - \theta^*\|_2^2 + 2 \gamma (\theta_{n} - \theta^*)\tr\left[f_{X_{i_n}}(\theta_n) - f_{X_{i_n}}(\bar\theta^{(m)}) + \E(f_{X_{i_n}}(\bar\theta^{(m)})\mid \F_n)\right] \nonumber\\
 &\quad+ \gamma^2 \left[\l f_{X_{i_n}}(\theta_n) - f_{X_{i_n}}(\bar\theta^{(m)}) + \E(f_{X_{i_n}}(\bar\theta^{(m)})\mid \F_n) \r^2\right] \label{eq:ctd-step1}
% =&\quad\|\theta_{n} - \theta^*\|_2^2 + 2 \gamma (\theta_{n} - \theta^*)\tr\E_{\Psi,\theta_n}(f_{X_{i_n}}(\theta_{n-1}) \\
% &+ \gamma^2 \E_{\Psi,\theta_n}\left[\l f_{X_{i_n}}(\theta_n) - f_{X_{i_n}}(\bar\theta^{(m)}) + \E_{\Psi,\theta_n}(f_{X_{i_n}}(\bar\theta^{(m)})) \r^2\right] + \gamma^2 \E_{\theta_n}\l \epsilon_n\r^2 \\
\end{align}
where $\Upsilon(\cdot)$ denotes the orthogonal projection onto the $H$-ball.
\paragraph{Step 2: Bounding the variance using the centering term.}\ \\
Let
\begin{align*}
\bar f_{X_{i_n}}(\theta_n) : =& f_{X_{i_n}}(\theta_n) - f_{X_{i_n}}(\bar\theta^{(m)}) + \E(f_{X_{i_n}}(\bar\theta^{(m)})\mid \F_n)\\
=& f_{X_{i_n}}(\theta_n) - f_{X_{i_n}}(\theta^*) 
- ( f_{X_{i_n}}(\bar\theta^{(m)}) -f_{X_{i_n}}(\theta^*)) + \E(f_{X_{i_n}}(\bar\theta^{(m)}) - f_{X_{i_n}}(\theta^*)\mid \F_n)\\
&+(\E(f_{X_{i_n}}(\theta^*)\mid \F_n) - \E_{\Psi,\theta_n}(f_{X_{i_n}}(\theta^*)))
\end{align*}
where we have used that $\E_{\Psi,\theta_n}(f_{X_{i_n}}(\theta^*)) = 0$.
Then, letting
\begin{align*}
e_n(\theta):= \E\left[\left. f_{X_{i_n}}(\theta)\right|\F_n\right] - \E_{\Psi,\theta_n}\left[ f_{X_{i_n}}(\theta)\right]
\end{align*}
we have,
\begin{align}
 \E&\left(\l \bar f_{X_{i_n}}(\theta_n) \r^2 \mid \F_n\right)\nonumber \\
   &\le  \E\bigg(\l f_{X_{i_n}}(\theta_n) - f_{X_{i_n}}(\theta^*)\r^2\nonumber \\
    &\qquad+ \l( f_{X_{i_n}}(\bar\theta^{(m)}) -f_{X_{i_n}}(\theta^*)) + \E(f_{X_{i_n}}(\bar\theta^{(m)}) - f_{X_{i_n}}(\theta^*)\mid \F_n)\r^2 + \l e_n(\theta^*) \r^2\mid\F_n\bigg)\nonumber\\
   &\le  \E\left(\left.\l f_{X_{i_n}}(\theta_n) - f_{X_{i_n}}(\theta^*)\r^2\right| \F_n\right) + \E\left(\left.\l f_{X_{i_n}}(\bar\theta^{(m)}) - f_{X_{i_n}}(\theta^*)\r^2\right| \F_n\right)+\E\left(\l e_n(\theta^*) \r^2\mid\F_n\right),\label{eq:a-appendix}
\end{align}
where we have used that for any random variable $X$, $\E(\| X - \E X\|^2) \le \E(\|X\|^2)$.

Let
\begin{align*}
\epsilon_n(\theta) = \E\left(\left.\l f_{X_{i_n}}(\theta) - f_{X_{i_n}}(\theta^*)\r^2\right| \F_n\right) - \E_{\Psi,\theta_n}(\l f_{X_{i_n}}(\theta) - f_{X_{i_n}}(\theta^*)\r^2).
\end{align*}
The second term in $\epsilon_n(\theta)$ can be bounded as follows:
For any $\theta$, we have
\begin{align}
 &\E_{\Psi,\theta_n}(\l f_{X_{i_n}}(\theta) - f_{X_{i_n}}(\theta^*)\r^2)\nonumber\\
 =& (\theta - \theta^*)\tr \E_{\Psi,\theta_n}\left[ (\beta\phi(s_{n+1})-\phi(s_{n}))\phi(s_{n})\tr \phi(s_{n})(\beta\phi(s_{n+1})-\phi(s_{n}))\tr \right](\theta - \theta^*)\nonumber\\
 =& (\theta - \theta^*)\tr \left(\Phi\tr(I - \beta P)\Psi \Phi \right)\tr \Phi\tr\Psi(I - \beta P) \Phi (\theta - \theta^*)\nonumber\\
 \le&  (\theta - \theta^*)\tr \left(\Phi\tr\Psi\Phi \right)\tr \Phi\tr\Psi\Phi (\theta - \theta^*)\nonumber\\
 \le& d^2\|\Phi (\theta - \theta^*)\|_{\Psi}^2.
\end{align}
In the final inequality, we have used $\l \Phi\tr \Psi \Phi\r \le d^2$.

Plugging the above in \eqref{eq:a-appendix}, we obtain
\begin{align*}
 \E\left(\left.\l \bar f_{X_{i_n}}(\theta_n)\r^2\right| \F_n\right)
   \le&  d^2\left( \|\Phi (\theta_n - \theta^*)\|_{\Psi}^2 + \|\Phi (\bar\theta^{(m)} - \theta^*)\|_{\Psi}^2\right)\\
   &+ \epsilon_n(\theta_n) + \epsilon_n(\bar\theta^{(m)}) +\E\left(\l e_n(\theta^*) \r^2\mid\F_n\right)
\end{align*}

\paragraph{Step 3: Analysis for a particular epoch.}\ \\
Using the last display, we calculate,
\begin{align}
\E&\left(\left. \|\theta_{n+1} - \theta^*\|_2^2\right| \F_n\right)
	\le\|\theta_{n} - \theta^*\|_2^2
		+ 2 \gamma (\theta_{n} - \theta^*)\tr\E\left[\left.\bar f_{X_{i_n}}(\theta_n)\right|\F_n\right] \nonumber\\
	&\qquad + \gamma^2 \bigg(d^2\left( \|\Phi (\theta_n - \theta^*)\|_{\Psi}^2
													+ \|\Phi (\bar\theta^{(m)} - \theta^*)\|_{\Psi}^2\right)
				+ \epsilon_n(\theta_n) + \epsilon_n(\bar\theta^{(m)})
													+ \E\left(\l e_n(\theta^*) \r^2\mid\F_n\right)\bigg) \nonumber\\
	&= \|\theta_{n} - \theta^*\|_2^2
		+ 2 \gamma (\theta_{n} - \theta^*)\tr\E\left[\left. f_{X_{i_n}}(\theta_n)\right|\F_n\right] \label{eq:a1}\\
	&\qquad + \gamma^2 \bigg(d^2\left( \|\Phi (\theta_n - \theta^*)\|_{\Psi}^2
													+ \|\Phi (\bar\theta^{(m)} - \theta^*)\|_{\Psi}^2\right)
		+ \epsilon_n(\theta_n) + \epsilon_n(\bar\theta^{(m)})
													+\E\left(\l e_n(\theta^*) \r^2\mid\F_n\right)\bigg). \nonumber
\end{align}
The equality above uses the fact that $\E\left[\left. \bar f_{X_{i_n}}(\theta_n)\right|\F_n\right] = \E\left[\left. f_{X_{i_n}}(\theta_n)\right|\F_n\right]$, since
$$f_{X_{i_n}}(\bar\theta^{(m)}) - \E\left[\left. f_{X_{i_n}}(\bar\theta^{(m)})\right|\F_n\right]$$
is a martingale difference w.r.t. $\F_n$.

We can rewrite \eqref{eq:a1} as
\begin{align}
\E&\left(\left. \|\theta_{n+1} - \theta^*\|_2^2\right| \F_n\right)
	\le\quad \|\theta_{n} - \theta^*\|_2^2
		+ 2 \gamma (\theta_{n} - \theta^*)\tr\E_{\Psi,\theta_n}\left[f_{X_{i_n}}(\theta_n)\right]
		+ 2 \gamma (\theta_{n} - \theta^*)\tr e_n(\theta_n)\nonumber\\
	&+ \gamma^2 \left(d^2\left( \|\Phi (\theta_n - \theta^*)\|_{\Psi}^2
											+ \|\Phi (\bar\theta^{(m)} - \theta^*)\|_{\Psi}^2\right)
											+ \epsilon_n(\theta_n) + \epsilon_n(\bar\theta^{(m)})
											+\E\left(\l e_n(\theta^*) \r^2\mid\F_n\right)\right) \label{eq:a2}
\end{align}

Notice that $\E_{\Psi,\theta_n}\left[f_{X_{i_n}}(\theta^*)\right]=0$ and hence
$$
\E_{\Psi,\theta_n}\left[f_{X_{i_n}}(\theta_n)\right] = \E_{\Psi,\theta_n}( (\beta\phi(s_{n+1})-\phi(s_{n}))\phi(s_{n})\tr) (\theta_n - \theta^*) = (\Phi\tr \Psi (I-\beta P) \Phi)(\theta_n-\theta^*).
$$ 

Using the above, we can simplify \eqref{eq:a2} as follows:
\begin{align}
&\E\left(\left. \|\theta_{n+1} - \theta^*\|_2^2\right| \F_n\right)
\nonumber\\
\le&\|\theta_{n} - \theta^*\|_2^2+ \gamma^2 d^2\left( \|\Phi (\theta_n - \theta^*)\|_{\Psi}^2 
																+ \|\Phi (\bar\theta^{(m)} - \theta^*)\|_{\Psi}^2\right)
	  - 2 \gamma (\theta_{n} - \theta^*)\tr\left[\Phi\tr \Psi (I-\beta P) \Phi\right]\nonumber\\
	& + 2 \gamma (\theta_{n} - \theta^*)\tr e_n(\theta_n)
	  + \gamma^2\epsilon_n(\theta_n) + \gamma^2\epsilon_n(\bar\theta^{(m)})
	  +\E\left(\l e_n(\theta^*) \r^2\mid\F_n\right)\nonumber\\
%%%%%%%%%%%%%%%%
\le& \|\theta_{n} - \theta^*\|_2^2
	- 2\gamma( (1-\beta) - \dfrac{d^2\gamma}{2})\|\Phi (\theta_n - \theta^*)\|_{\Psi}^2 
	+ \gamma^2 d^2 \|\Phi (\bar\theta^{(m)} - \theta^*)\|_{\Psi}^2\nonumber\\
	& + 2 \gamma (\theta_{n} - \theta^*)\tr e_n(\theta_n)
	+ \gamma^2\epsilon_n(\theta_n) + \gamma^2\epsilon_n(\bar\theta^{(m)})
	+\E\left(\l e_n(\theta^*) \r^2\mid\F_n\right) \label{eq:a3}
\end{align}
Unrolling the above recursion within an epoch, i.e., from $n=(m+1)M -1$ to $n=mM$, we obtain
\begin{align}
\E&\left(\left. \|\theta_{(m+1)M-1} - \theta^*\|_2^2\right| \F_n\right)
	\le\|\theta_{mM} - \theta^*\|_2^2 \nonumber\\
		&- \sum_{n=mM}^{(m+1)M -2} 2\gamma( (1-\beta) - \dfrac{d^2\gamma}{2})
							\E\left(\left.\|\Phi (\theta_n - \theta^*)\|_{\Psi}^2\right| \F_{mM-1}\right)
		  + M \gamma^2 d^2 \|\Phi (\bar\theta^{(m)} - \theta^*)\|_{\Psi}^2\nonumber\\
		&+ \E\left(\left.
				\sum_{n=mM}^{(m+1)M -2} 2 \gamma (\theta_{n} - \theta^*)\tr e_n(\theta_n)
								+ \gamma^2\left(\epsilon_n(\theta_n)+ \epsilon_n(\bar\theta^{(m)})
								+\E\left(\l e_n(\theta^*) \r^2\mid\F_n\right)\right) \right| \F_{mM-1}\right) \label{eq:a4}
\end{align}

Notice that 
$(\bar\theta^{(m)} - \theta^*)\tr I (\bar\theta^{(m)} - \theta^*) \le \dfrac{1}{\mu}(\bar\theta^{(m)} - \theta^*)\tr \Phi\tr \Psi \Phi (\bar\theta^{(m)} - \theta^*)$ and hence we obtain the following by setting $\theta_{mM} = \bar\theta^{(m)}$ and ignoring the non-negative LHS term:
\begin{align*}
2 \gamma M & ((1-\beta) - \dfrac{d^2 \gamma}{2})
	\E\left(\left. \|\Phi (\bar\theta^{(m+1)} - \theta^*)\|_{\Psi}^2 \right| \F_{mM}\right) \\
\le&\left(\dfrac{1}{\mu} + M\gamma^2 d^2\right) \|\bar\theta^{(m)} - \theta^*\|_2^2
		+ \gamma^2\sum_{n=mM}^{(m+1)M -2}
			\E\left(\left.\epsilon_n(\theta_n) + \epsilon_n(\bar\theta^{(m)})
			+\l e_n(\theta^*) \r^2 \right| \F_{mM}\right)\\ 
	& + \E\left(\left.2 \gamma \sum_{n=mM}^{(m+1)M -1}(\theta_{n} - \theta^*)\tr e_n(\theta_n)\right| \F_{mM}\right)
\end{align*}

 \paragraph{Step 4: Combining across epochs.}\ \\
Let $C_1:=\left(\dfrac{1}{2 \mu\gamma M ((1-\beta) - \dfrac{d^2\gamma}{2})} + \dfrac{\gamma d^2}{2 ((1-\beta) - \dfrac{d^2\gamma}{2})}\right)$. Then, we have
\pagebreak
\begin{align*}
%%%%%%%%%%%%%%%%%%%%%Here
&\E\left(\left. \|\Phi (\bar\theta^{(m)} - \theta^*)\|_{\Psi}^2 \right| \F_{0}\right) \\
	\le&\quad  C_1^m \|\bar\theta^{(0)} - \theta^*\|_2^2 + \frac{\gamma C_2}{2} \sum_{k=1}^{m-1} C_1^{(m-2-k)}
				\sum_{i=(k-1)M}^{kM -1}\E\left(\left.\epsilon_i(\theta_i) + \epsilon_i(\bar\theta^{((k-1)M)})
																	+\l e_i(\theta^*) \r^2\right| \F_{0}\right)\\ 
		&+  C_2 \sum_{k=1}^{m-1} C_1^{(m-2-k)}
				\sum_{i=(k-1)M}^{kM -1} \E\left(\left.(\theta_{i} - \theta^*)\tr e_i(\theta_i)\right| \F_{0}\right),
\end{align*}
where $C_2= \dfrac{\gamma}{2 M ((1-\beta) - \dfrac{d^2\gamma}{2})}$.

 \paragraph{Step 5: Controlling the error terms.}\ \\
 First note that
 \begin{align*}
 \epsilon(\theta)
 	=& \E\left( \l (\beta\phi(s_n)\phi(s_{n+1})\tr - \phi(s_n)\phi(s_n))(\theta - \theta^*)\tr\r^2\mid\F_n\right)\\
 	&- \E_{\Psi,\theta_n}\left( \l (\beta\phi(s_n)\phi(s_{n+1})\tr - \phi(s_n)\phi(s_n))(\theta - \theta^*)\tr\r^2\right)\\
 	=& (\theta - \theta^*)\tr\left[ \E(v\tr v\mid \F_n) - \E_{\Psi,\theta_n}(v\tr v) \right](\theta - \theta^*)
 \end{align*}
 where $v = \beta\phi(s_n)\phi(s_{n+1})\tr - \phi(s_n)\phi(s_n)$. Supposing that $\l (\theta - \theta^*) \r\le H$, and using that $\l \phi(s_n) \r \le 1$ together with the convexity of matrix inner products in the $\l\cdot\r$-matrix norm, we have that
 \begin{align*}
 \l\epsilon(\theta)\r
 	\le 2H \l \E(v\mid \F_n) - \E_{\Psi,\theta_n}(v) \r
 \end{align*}
 So by (A6), and the fact the projection step of the algorithm, we have 
 \begin{align}\label{eqn:epsilon_error}
 \sum_{i=(k-1)M}^{kM -1}
 	\E\left(\left.\l\epsilon_i(\theta_i) + \epsilon_i(\bar\theta^{((k-1)M)}) \r^2\right| \F_{0}\right) \le 8HB_{(k-1)M}^{kM}(s_0).
 \end{align}
 Similarly
 \begin{align*}
 \sum_{i=(k-1)M}^{kM -1}\E\left(\left.\l e_i(\theta^*) \r^2\right| \F_{0}\right),
 \sum_{i=(k-1)M}^{kM -1} \E\left(\left.(\theta_{i} - \theta^*)\tr e_i(\theta_i)\right| \F_{0}\right)
 \le 2HB_{(k-1)M}^{kM}(s_0)
 \end{align*}
\end{proof}

%%%%%%%%%%%%%%%%%%%%%%%%%%%%%%%%%%%%%%%%%%%%%%%%%%%%%%%%%%%%%%
%%%%%%%%%%%%%%%%%%%%%%%%%%%%%%%%%%%%%%%%%%%%%%%%%%%%%%%%%%%%%%
%%%%%%%%%%%%%%%%%%%%%%%%%%%%%%%%%%%%%%%%%%%%%%%%%%%%%%%%%%%%%%
\section{Numerical Experiments}
\label{sec:experiments}
\newcommand{\errorband}[5][]{ % x column, y column, error column, optional argument for setting style of the area plot
\pgfplotstableread[col sep=comma, skip first n=2]{#2}\datatable
    % Lower bound (invisible plot)
    \addplot [draw=none, stack plots=y, forget plot] table [
        x={#3},
        y expr=\thisrow{#4}-\thisrow{#5}
    ] {\datatable};

    % Stack twice the error, draw as area plot
    \addplot [draw=none, fill=gray!40, stack plots=y, area legend, #1] table [
        x={#3},
        y expr=2*\thisrow{#5}
    ] {\datatable} \closedcycle;

    % Reset stack using invisible plot
    \addplot [forget plot, stack plots=y,draw=none] table [x={#3}, y expr=-(\thisrow{#4}+\thisrow{#5})] {\datatable};
}
 \begin{figure}
    \centering
    \begin{tabular}{c}
    \subfigure[Example 1]
    {
      \tabl{c}{\scalebox{1.0}{\begin{tikzpicture}
      \begin{axis}[
        width=10cm,
        height=9cm,
	xlabel={iterations},
	ylabel={Normalized value diff},
       legend entries={
	 ,
        TD(0),
        ,
        TD(0)-Avg
        },
        legend pos=north east,
      ]

      % TD
      \errorband[red!90!white, opacity=0.3]{results/Exp1Results_filtered.csv}{0}{1}{2}
      \addplot [thick, red] table [x index=0, y index=1,col sep=comma] {results/Exp1Results_filtered.csv};
      % Averaging
      \errorband[blue!90!white, opacity=0.3]{results/Exp1Results_filtered.csv}{0}{3}{4}
      \addplot [thick, blue] table [x index=0, y index=3,col sep=comma] {results/Exp1Results_filtered.csv};
      \end{axis}
      \end{tikzpicture}}\\}
      \label{fig:example1} 
    }
    \\
        \subfigure[Example 2]
    {
      \tabl{c}{\scalebox{1.0}{\begin{tikzpicture}
      \begin{axis}[
	width=10cm,
        height=9cm,
        xlabel={iterations},
	ylabel={Normalized value diff},
       legend entries={
	 ,
        TD(0),
        ,
        TD(0)-Avg,
        ,
        CTD
        },
        legend pos=north east,
      ]

      % TD
      \errorband[red!90!white, opacity=0.3]{results/Exp2Results_subsampled100.csv}{0}{1}{2}
      \addplot [thick, red] table [x index=0, y index=1,col sep=comma] {results/Exp2Results_subsampled100.csv};
      % Averaging
      \errorband[blue!70!white, opacity=0.3]{results/Exp2Results_subsampled100.csv}{0}{3}{4}
      \addplot [thick, blue] table [x index=0, y index=3,col sep=comma] {results/Exp2Results_subsampled100.csv};
      % CTD
      \errorband[green!70!white, opacity=0.3]{results/Exp2Results_subsampled100.csv}{0}{5}{6}
      \addplot [thick, green!50!black] table [x index=0, y index=5,col sep=comma] {results/Exp2Results_subsampled100.csv};
      \end{axis}
      \end{tikzpicture}}\\}
      \label{fig:example2} 
    }
\end{tabular}
     \caption{Empirical illustration of TD(0), TD(0) with averaging and CTD algorithms. The normalised value difference is defined to be $\| \Phi(\theta_n - \theta^*)\|_\Psi/\|\Phi(\theta^*)\|_\Psi$.}
    \label{fig:all_results}
\end{figure}
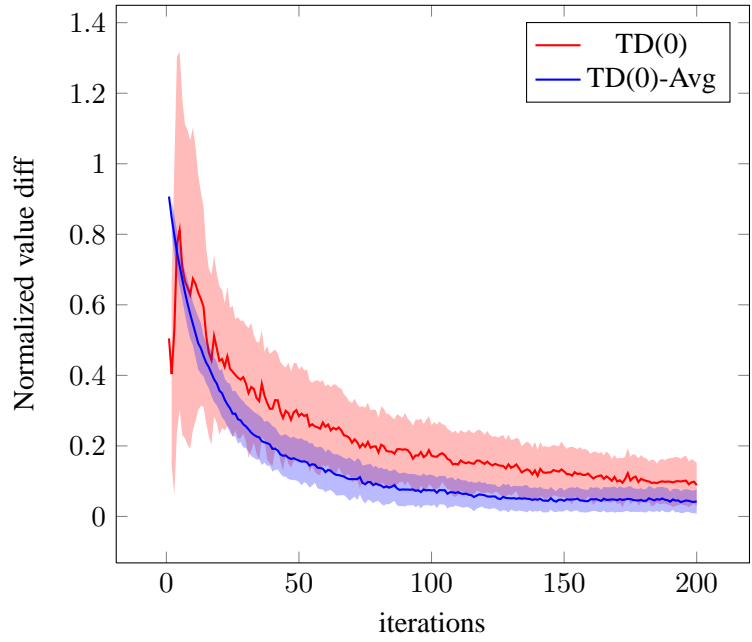
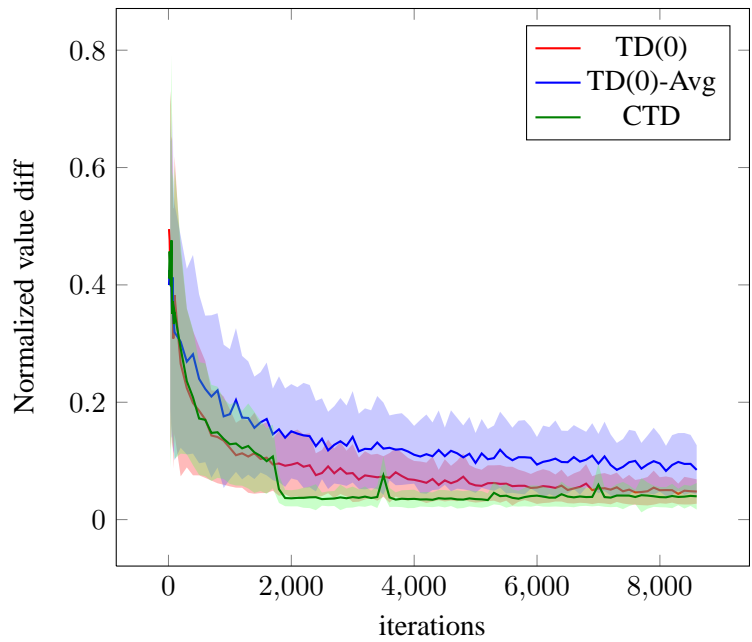
We test the performance of TD(0), TD(0) with averaging and CTD algorithms. 
\paragraph{Example 1.} This is a two-state toy example, which is borrowed from \citep{yu2009convergence}. The setting has the transition structure $P= \left[ 0.2, 0.8; 0.3,0.7\right]$
%$P = \left[
%\begin{array}{cc}
%0.2 & 0.8\\
%0.3 & 0.7
%\end{array}
%\right]
%$
and the rewards given by $r(1,j)=1, r(2,j)=2$, for $j=1,2$.  The features are one-dimensional, i.e., $\Phi = (1\quad 2)\tr$. 

Fig. \ref{fig:example1} presents the results obtained on this example. For setting the step-sizes of TD(0), we used the guideline from Theorem \ref{thm:td-rate}. Note that this results in convergence for TD(0), with the caveat that setting the step-size constant $c$ requires knowledge of underlying transition structure through $\mu$. It is evident that TD(0) with averaging gives performance on par with TD(0) and unlike TD(0), the setting of $c$ is not constrained here.
Given that convergence is rapid for TD(0) on this example, we do not plot CTD in Fig \ref{fig:example1} as the epoch length suggested by Theorem \ref{thm:ctd-bound} is $100$ and this is already enough for TD(0) itself to converge. CTD resulted in a normalized value difference of about $0.03$ on this example, but the effect of averaging across epochs for CTD will be seen better in the next example.

\paragraph{Example 2.} Here the number of states are $100$, the transitions are governed by a random stochastic matrix and the rewards are random and bounded between $0$ and $1$. Features are $3$-dimensional and are picked randomly in $(0,1)$. The  results obtained for the three algorithms are presented in Fig. \ref{fig:example2}. It is evident that all algorithms converge, with CTD showing the lowest variance. As in example 1, the setting parameters for TD(0) was dictated by Theorem \ref{thm:td-rate}, while for CTD, the step-size and epoch length were set such that the constant $C_1$ in Theorem \ref{thm:ctd-bound} is less than $1$.
%%%%%%%%%%%%%%%%%%%%%%%%%%%%%%%%%%%%%%%%%%%%%%%%%%%%%%%%%%%%%%
%%%%%%%%%%%%%%%%%%%%%%%%%%%%%%%%%%%%%%%%%%%%%%%%%%%%%%%%%%%%%%
%%%%%%%%%%%%%%%%%%%%%%%%%%%%%%%%%%%%%%%%%%%%%%%%%%%%%%%%%%%%%%
\section{Conclusions}
\label{sec:conclusions}
TD(0) with linear function approximators is a well-known policy evaluation algorithm.
While asymptotic convergence rate results are available for this algorithm, there are no finite-time bounds that quantify the rate of convergence. In this paper, we derived non-asymptotic bounds, both in high-probability as well as in expectation. From our results, we observed that iterate averaging is necessary to obtain the optimal $O\left(1/\sqrt{n}\right)$ rate of convergence. This is because, to obtain the optimal rate with the classic step-size choice $\propto 1/n$, it is necessary to know properties of the stationary distribution of the underlying Markov chain. We also proposed a fast variant of TD(0) that incorporates a centering sequence and established that the rate of convergence of this algorithm is exponential. We established the practicality of our bounds by using them to guide the step-size choices in two synthetic experimental setups.
%%%%%%%%%%%%%%%%%%%%%%%%%%%%%%%%%%%%%%%%%%%%%%%%%%%%%%%%%%%%%%
%%%%%%%%%%%%%%%%%%%%%%%%%%%%%%%%%%%%%%%%%%%%%%%%%%%%%%%%%%%%%%
%%%%%%%%%%%%%%%%%%%%%%%%%%%%%%%%%%%%%%%%%%%%%%%%%%%%%%%%%%%%%%
%%%%%%%%%%%%%%%%%%%%%%%%%%%%%%%%%%%%%%%%%%%%%%%%%%%%%%%%%%%%%%
%%%%%%%%%%%%%%%%%%%%%%%%%%%%%%%%%%%%%%%%%%%%%%%%%%%%%%%%%%%%%%

%\paragraph{\bf Acknowledgments}
%The first author was gratefully supported by the EPSRC Autonomous Intelligent Systems project EP/I011587.
%The second author would like to thank the European Community's Seventh Framework Programme (FP$7/2007-2013$) under grant agreement n$^o$ $270327$ for funding the research leading to these results.

%%%%%%%%%%%%%%%%%%%%%%%%%%%%%%%%%%%%%%%%%%%%%%%%%%%%%%%%%%%%%%
%%%%%%%%%%%%%%%%%%%%%%%%%%%%%%%%%%%%%%%%%%%%%%%%%%%%%%%%%%%%%%
%%%%%%%%%%%%%%%%%%%%%%%%%%%%%%%%%%%%%%%%%%%%%%%%%%%%%%%%%%%%%%
%%%%%%%%%%%%%%%%%%%%%%%%%%%%%%%%%%%%%%%%%%%%%%%%%%%%%%%%%%%%%%
%%%%%%%%%%%%%%%%%%%%%%%%%%%%%%%%%%%%%%%%%%%%%%%%%%%%%%%%%%%%%%

\clearpage
\newpage
\bibliographystyle{plainnat}
\bibliography{sgd}

%%%%%%%%%%%%%%%%%%%%%%%%%%%%%%%%%%%%%%%%%%%%%%%%%%%%%%%%%%%
%%%%%%%%%%%%%%%%%%%%%%%%%%%%%%%%%%%%%%%%%%%%%%%%%%%%%%%%%%%
%%%%%%%%%%%%%%%%%%%%%%%%%%%%%%%%%%%%%%%%%%%%%%%%%%%%%%%%%%%
%%%%%%%%%%%%%%%%%%%%%%%%%%%%%%%%%%%%%%%%%%%%%%%%%%%%%%%%%%%
%%%%%%%%%%%%%%%%%%%%%%%%%%%%%%%%%%%%%%%%%%%%%%%%%%%%%%%%%%%
\end{document}